\documentclass[journal]{IEEEtran}
\ifCLASSINFOpdf
\else
\fi
\usepackage{amsmath,amsthm,amsfonts,mathtools}
\usepackage{algorithmic}
\usepackage{array}
\usepackage[ruled,vlined]{algorithm2e}
\ifCLASSOPTIONcompsoc
  \usepackage[caption=false,font=normalsize,labelfont=sf,textfont=sf]{subfig}
\else
  \usepackage[caption=false,font=footnotesize]{subfig}
\fi
\usepackage[normalem]{ulem}
\usepackage{wrapfig}
\usepackage{fixltx2e,enumitem}
\makeatletter
\def\env@cases{%
  \let\@ifnextchar\new@ifnextchar
  \left\lbrace
  \def\arraystretch{1}%
  \array{@{}l@{\quad}l@{}}}
\makeatother
\usepackage{stfloats}
\usepackage{color, cite}
\usepackage{amssymb,soul}
\makeatletter
\def\thm@space@setup{\thm@preskip=2.2pt
\thm@postskip=0pt}
\makeatother
\newtheorem{theorem}{Theorem}
\newtheorem{lemma}{Lemma}
\newtheorem{proposition}{Proposition}
\theoremstyle{definition}

\newtheorem{assumption}{Assumption}
\theoremstyle{remark}
\newtheorem{remark}{Remark}

\usepackage{comment}
\usepackage{xr}
\makeatletter
\newcommand*{\addFileDependency}[1]{
  \typeout{(#1)}
  \@addtofilelist{#1}
  \IfFileExists{#1}{}{\typeout{No file #1.}}
}
\makeatother

\newcommand*{\myexternaldocument}[1]{%
    \externaldocument{#1}%
    \addFileDependency{#1.tex}%
    \addFileDependency{#1.aux}%
}
\myexternaldocument{SuppMat}
\usepackage{graphicx}
\graphicspath{{./images/}}
\ifCLASSOPTIONcaptionsoff
  \usepackage[nomarkers]{endfloat}
 \let\MYoriglatexcaption\caption
 \renewcommand{\caption}[2][\relax]{\MYoriglatexcaption[#2]{#2}}
\fi
\usepackage{url}
\newcounter{appendixx}
\newenvironment{appendixx}[1][]{\refstepcounter{appendixx}
   \textbf{Appendix~\theappendixx.} \rmfamily}{\medskip}
\renewcommand{\theappendixx}{\Alph{appendixx}}

\begin{document}
\onecolumn
\pagenumbering{gobble}
\setlist[itemize]{leftmargin=*}

\twocolumn
\title{Parameterized MDPs and Reinforcement Learning Problems - A Maximum Entropy Principle Based Framework}
\author{Amber~Srivastava and~Srinivasa~M~Salapaka,~\IEEEmembership{Member,~IEEE}
\vspace{-0.8cm}
\thanks{The paper has been accepted for publication in IEEE Transactions on Cybernetics. The Institute of Electrical and Electronics Engineers, Incorporated (the "IEEE") holds all rights under copyright.}
\thanks{This work was supported by NSF grant ECCS (NRI) 18-30639, DOE award DE-EE0009125, and Dynamic Research Enterprise for Multidisciplinary Engineering Sciences (DREMES) - collaboration between Zhejiang University and the University of Illinois at Urbana-Champaign.}
\thanks{The authors are with the Mechanical Science and Engineering Department and Coordinated Science Laboratory, University of Illinois at Urbana-Champaign, IL, 61801 USA. E-mail: \{asrvstv6, salapaka\}@illinois.edu.}}
\maketitle
\begin{abstract}
We present a framework to address a class of sequential decision making problems. Our framework features learning the optimal control policy with robustness to noisy data, determining the unknown state and action parameters, and performing sensitivity analysis with respect to problem parameters. We consider two broad categories of sequential decision making problems modeled as infinite horizon Markov Decision Processes (MDPs) with (and without) an absorbing state. The central idea underlying our framework is to quantify exploration in terms of the Shannon Entropy of the trajectories under the MDP and determine the stochastic policy that maximizes it while guaranteeing a low value of the expected cost along a trajectory. This resulting policy enhances the quality of exploration early on in the learning process, and consequently allows faster convergence rates and robust solutions even in the presence of noisy data as demonstrated in our comparisons to popular algorithms such as Q-learning, Double Q-learning and entropy regularized Soft Q-learning. The framework extends to the class of parameterized MDP and RL problems, where states and actions are parameter dependent, and the objective is to determine the optimal parameters along with the corresponding optimal policy. Here, the associated cost function can possibly be non-convex with multiple poor local minima. Simulation results applied to a 5G small cell network problem demonstrate successful determination of communication routes and the small cell locations. We also obtain sensitivity measures to problem parameters and robustness to noisy environment data.
\end{abstract}
\section{Introduction}\label{sec: Intro}
\IEEEPARstart{M}{arkov} Decision Processes (MDPs) model sequential decision making problems which arise in many application areas such as robotics, sensor networks, economics, and manufacturing. These models are characterized by the state-evolution  dynamics $s_{t+1}=f(s_t,a_t)$, a control policy $\mu(a_t|s_t)$ that allocates an {\em action} $a_t$ from a control set to each state $s_t$, and a cost $c(s_t,a_t,s_{t+1})$ associated with the transition from $s_t$ to $s_{t+1}$. The goal in these applications is to determine the optimal control policy that results in a {\em path}, a sequence of actions and states, with minimum cumulative cost. There are many variants of this problem \cite{feinberg2012handbook}, where the dynamics can be defined over finite or infinite horizons; where the state-dynamics $f$ can  be stochastic; where the models for the state dynamics may be partially or completely unknown, and the cost function is not known a priori, albeit the cost at each step is revealed at the end of each transition. Some of the most common methodologies that address MDPs include dynamic programming, value and policy iterations \cite{bertsekas1996neuro}, linear programming \cite{hordijk1979linear,abbasi2014linear}, and Q-learning \cite{watkins1992q}.

In this article, we view MDPs and their variants as combinatorial optimization problems, and develop a framework based on maximum entropy principle (MEP) \cite{jaynes1957information} to address them. MEP has proved successful in addressing a variety of combinatorial optimization problems such as facility location problems  \cite{rose1991deterministic}, combinatorial drug discovery \cite{sharma2008scalable}, traveling salesman problem and its variants \cite{rose1991deterministic}, image processing \cite{yu2013maximal}, graph and markov chain aggregation \cite{xu2014aggregation}, and protein structure alignment \cite{chen2005protein}. 
MDPs, too, can be viewed as {\em combinatorial} optimization problems - due to the combinatorially large number of paths (sequence of consecutive states and actions) that it may take based on the control policy and its inherent stochasticity. In our MEP framework, we determine  a probability distribution defined on {\em the space of paths} \cite{ziebart2008maximum}, such that (a) it is the {\em fairest} distribution - the one with the maximum Shannon Entropy $H$, and  (b) it satisfies the constraint that the expected cumulative cost $J$ attains a prespecified feasible value $J_0$.  The framework results in an iterative scheme, an {\em annealing} scheme, where probability distributions are improved upon by successively lowering the prespecified values $J_0$. In fact, the lagrange multiplier $\beta$ corresponding to the cost constraint ($J=J_0$) in the unconstrained lagrangian is increased from small values (near 0) to large values to effect annealing. Higher values of multipliers correspond to lower values of the expected cost. We show that as multiplier value increases, the corresponding probability distributions become more localized, finally  converging to a deterministic policy.   

This framework is applicable to all the classes of MDPs and its variants described above. Our MEP based approach inherits the flexibility of algorithms such as deterministic annealing \cite{rose1991deterministic} developed in the context of combinatorial resource allocation, which include adding capacity, communication, and dynamic  constraints. The added advantage of our approach is that we can  draw close parallels to existing algorithms for MDPs and RL (e.g. Q-Learning) – thus enabling us to exploit their algorithmic insights. Below we highlight main contributions and advantages of our approach.

{\em Exploration and Unbiased Policy: } In the context of model-free RL setting, the algorithms interact with the {\em environment} via {\em agents} and rely upon the instantaneous cost (or reward) generated by the environment to learn the optimal policy. Some of the popular algorithms include Q-learning \cite{watkins1992q}, Double Q-learning \cite{hasselt2010double}, Soft Q-learning (entropy regularized Q-learning) \cite{fox2015taming,grau2018soft,peters2010relative,neu2017unified,asadi2017alternative,nachum2017bridging,dai2017sbeed} in discrete state and action spaces, and Trust Region Policy Optimization (TRPO) \cite{schulman2015trust}, and Soft Actor Critic (SAC) \cite{haarnoja2018soft} in continuous spaces. It is commonly known that for the above algorithms to perform well, all {\em relevant} states and actions should be explored. In fact, under the assumption that each state-action pair is visited multiple times during the learning process it is guaranteed that the above discrete space algorithms \cite{watkins1992q,hasselt2010double,fox2015taming,grau2018soft} will converge to the optimal policy. Thus, the adequate {\em exploration} of the state and action spaces becomes incumbent to the success of these algorithms in determining the optimal policy. Often the instantaneous cost is noisy \cite{fox2015taming} which hinders with the learning process and demands for an enhanced quality exploration.

In our MEP based approach, the Shannon Entropy of the probability distribution over the paths in the MDP explicitly characterizes the {\em exploration}.  The framework results in a {\em distribution over the paths} that is as {\em unbiased} as possible under the given cost constraint. The corresponding stochastic policy is maximally noncommittal to any particular path in the MDP that achieves the constraint; this results in better (unbiased) exploration. The policy starts from being entirely explorative, when multiplier value is small ($\beta\approx 0$),  and becomes increasingly {\em exploitative} as the multiplier value increases.

{\em Parameterized MDPs and RL: } These class of optimization problems are not even necessarily MDPs which contributes significantly to their inherent complexities. However, we model them in a specific way to retain the Markov property without any loss of generality, thereby making these problems tractable. Scenarios such as self organizing networks \cite{aguilar2015location}, 5G small cell network design \cite{siddique2015wireless,manganini2015policy}, supply chain networks, and last mile delivery problems \cite{9096570} pose a {\em parameterized} MDP with a two-fold objective of determining simultaneously  (a) the optimal control policy for the underlying stochastic process, and (b)  the unknown parameters that the state and action variables depend upon such that the cumulative cost is minimized. The latter objective is akin to facility location problem \cite{mahajan2009planar,abin2016querying,huang2017locally}, that is shown to be NP-hard \cite{mahajan2009planar}, and where the associated cost function (non-convex) is riddled with multiple poor local minima. \begin{wrapfigure}{r}{0.6\columnwidth}
\centering
\includegraphics[width=0.6\columnwidth]{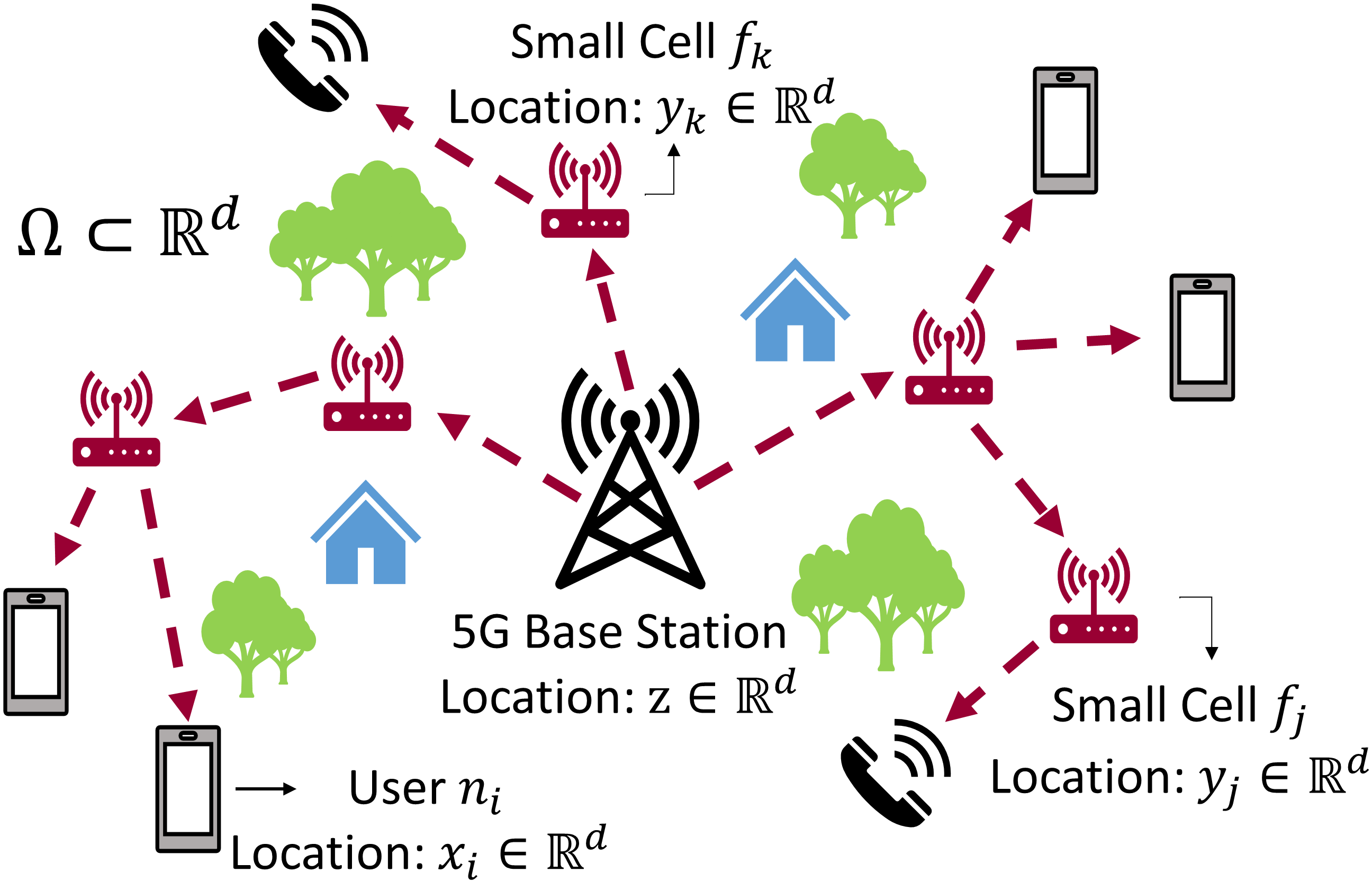}
\caption{Illustrates the 5G Small Cell Network. The objective is to determine the small cell location $\{y_j\in\mathbb{R}^d\}$ and the communication routes from the Base Station $\delta$ to each user $\{n_i\}$ via the network of the small cells.}
\label{fig: 5G_smallCell}
\vspace{-0.4cm}
\end{wrapfigure} For instance, Figure \ref{fig: 5G_smallCell} illustrates a 5G small cell network, where the objective is to simultaneously determine the locations of the small cells $\{f_j\}$ and design the communication paths (control policy) between the user nodes $\{n_i\}$ and base station $\delta$ via network of small cells. Here, the state space $\mathcal{S}$ of the underlying MDP is parameterized by the locations $\{y_j\}$ of small cells $\{f_j\}$.

{\em Algebraic structure and Sensitivity Analysis:} In our framework, maximization of Shannon entropy of the distribution over the paths under constraint on the cost function value results in an unconstrained Lagrangian - the {\em free-energy} function. This function is a {\em smooth} approximation of the cumulative cost function of the MDP, which  enables the use of calculus. We exploit this distinctive feature of our framework to determine the unknown state and action parameters in case of parameterized MDPs and perform sensitivity analysis for various problem parameters. Also, the framework easily accommodates stochastic models that describe uncertainties in the instantaneous cost and parameter values.

{\em Algorithmic Guarantees and Innovation:} For the classes of MDPs that we consider, our MEP-based framework results into non-trivial derivations of the recursive Bellman equation for the associated Lagrangian. We show that these Bellman operators are contraction maps and use their several properties to guarantee the convergence to the optimal policy and as well as to a local minima in the case of parameterized MDPs.

In the context of model-free RL, we provide comparisons with the benchmark algorithms Q, Double Q, and entropy regularized G-learning \cite{fox2015taming} (also referred to as Soft Q-learning). Our algorithms converge at a faster rate (as fast as $1.5$ times) than the benchmark algorithms across various values of discount factor, and even in the case of noisy environments. In the context of parameterized MDPs and RL, we address the small-cell network design problem in 5G communication. Here the parameters are the unknown locations of the small-cells and the control policy determines the routing of the communication packet. Upon comparison with the sequential method of first determining the unknown parameters (small cell locations) and then the control policy (equivalently, the communication paths), we show that our algorithms result into costs that are as low as $65\%$ of the former. The efficacy of our algorithms can be assessed from the fact that the solutions in the model-based and model-free cases are nearly the same. We also demonstrate sensitivity analysis, benefits of annealing  and considering entropy of distribution over the paths in our simulations on parameterized MDPs and RL. 

This paper is organized as follows. We briefly review the related work and MEP \cite{jaynes1957information} in the Section \ref{sec: MEP_Prelim}. In Section \ref{sec: gen_inst} and \ref{sec: InfiMDP} we develop MEP-based framework for MDPs. Section \ref{sec: ParMDP} builds upon the Section \ref{sec: gen_inst} to address the case of parameterized MDPs and RL problems. Simulations on a variety of scenarios are presented in Section \ref{sec: Simulations}. We discuss the generality of our framework, its capabilities, and future directions of the work in Section \ref{sec: Ana_Disc}. For the ease of reading, we provide a comprehensive list of symbols in the Section \ref{LOS} of the supplementary material.

\section{Preliminaries}\label{sec: MEP_Prelim}
{\em Related Work in Entropy Regularization: }Some of the previous works in RL literature \cite{fox2015taming, grau2018soft, peters2010relative, neu2017unified, asadi2017alternative, nachum2017bridging,dai2017sbeed,shi2019soft,xiang2019task} either use entropy as a regularization term ($-\log \mu(a_t|s_t)$) \cite{fox2015taming, grau2018soft} to the instantaneous cost function $c(s_t,a_t,s_{t+1})$ or maximize the entropy ($-\sum_a \mu(a|s)\log\mu(a|s)$) \cite{peters2010relative, neu2017unified,asadi2017alternative} associated {\em only} to the stochastic policy under constraints on the cost $J$. This results into benefits such as better exploration, overcoming the effect of noise $w_t$ in the instantaneous cost $c_t$ and obtaining faster convergence. However, the resulting stochastic policy and soft-max approximation of the value function $J$ is not in compliance with the Maximum Entropy Principle applied to the distribution over the paths of MDP. Thus, the resulting stochastic policy is biased in its exploration {\em over the paths of the MDP}. Our simulations demonstrate the benefit of {\em unbiased} exploration (in our framework) in terms of faster convergence and better performance in noisy environment in comparison to the entropy regularized benchmark algorithm.

{\em Related Work in Parameterized MDPs and RL: }The existing solution approaches \cite{ hordijk1979linear, bertsekas1996neuro,abbasi2014linear} can be extended to the parameterized MDPs by discretizing the parameter domain. However, the resulting problem is not necessarily an MDP as every transition from one state to another is dependent on the path (and the parameter values) taken till the current state. Other related approaches for parameterized MDPs are case specific; for  instance,  \cite{xia2015parameterized} presents action-based parameterization of state space with application to service rate control in closed Jackson networks, and \cite{hausknecht2015deep, masson2016reinforcement, wei2018hierarchical, xiong2018parametrized,narayanan2017event,ccilden2014toward} incorporate parameterized actions that is applicable in the domain of RoboCup soccer where at each step the agent must select both the discrete action it wishes to execute as well as continuously valued parameters required by that action. On the other hand, the class of parameterized MDPs that we address in this article predominantly originate in network based applications that involves simultaneous routing and resource allocations and pose additional challenges of non-convexity and NP-hardness. We address these MDPs in both the scenarios, where the underlying model is known as well as unknown.

{\em Maximum Entropy Principle: }
We briefly review the Maximum Entropy Principle (MEP) \cite{jaynes1957information} since our framework relies heavily upon it. MEP states that for a random variable $\mathcal{X}$ with a given prior information, the most unbiased probability distribution given prior data is the one that maximizes the Shannon entropy. More specifically, let the known prior information of the random variable $\mathcal{X}$ be given as constraints on the expectation of the functions $f_k:\mathcal{X}\rightarrow \mathbb{R}$, $1\leq k\leq m$. Then, the most unbiased probability distribution $p_{\mathcal{X}}(\cdot)$ solves
\begin{align}\label{eq: optim_MEP}
\begin{split}
\max_{\{p_{\mathcal{X}}(x_i)\}}\quad & H(\mathcal{X}) = -\sum_{i=1}^n p_{\mathcal{X}}(x_i) \ln p_{\mathcal{X}}(x_i)\\
\text{subject to}\quad & \sum_{i=1}^n p_{\mathcal{X}}(x_i) f_k(x_i) = F_k~~\forall ~1\leq k\leq m,
\end{split}
\end{align}
where $F_k$, $1\leq k\leq m$, are known expected values of the functions $f_k$. The above optimization problem results into the Gibbs' distribution  \cite{jaynes2003probability} $p_{\mathcal{X}}(x_i) = \frac{\exp\{-\sum_k \lambda_k f_k(x_i)\}}{\sum_{j=1}^n \exp\{-\sum_k\lambda_k f_k(x_j)\}}$, where $\lambda_k$, $1\leq k\leq m$, are the Lagrange multipliers corresponding to the inequality constraints in (\ref{eq: optim_MEP}).

\section{MDPs with Finite Shannon Entropy}\label{sec: gen_inst}
\subsection{Problem Formulation}\label{sec: MDP}
We consider an infinite horizon discounted MDP that comprises of a {\em cost-free termination} state $\delta$. We formally define this MDP as a tuple $\langle \mathcal{S}, \mathcal{A}, c,p, \gamma\rangle$ where  $\mathcal{S}=\{s_1,\hdots,s_N=\delta\}$, $\mathcal{A}=\{a_1,\hdots,a_M\}$, and $c:\mathcal{S}\times\mathcal{A}\times\mathcal{S}\rightarrow \mathbb{R}$ respectively denote the state space, action space, and cost function; $p:\mathcal{S}\times\mathcal{S}\times\mathcal{A}\rightarrow [0,1]$ is the state transition probability function and $0 < \gamma \leq 1$ is the discounting factor. A control policy $\mu:\mathcal{A}\times\mathcal{S}\rightarrow \{0,1\}$ determines the action taken at each state $s\in\mathcal{S}$, where $\mu(a|s) = 1$ implies that action $a\in\mathcal{A}$ is taken when the system is in the state $s\in\mathcal{S}$ and $\mu(a|s)=0$ indicates otherwise. For every initial state $x_0=s$, the MDP induces a stochastic process, whose realization is a  {\em path} $\omega$ - an infinite sequence of actions and states, that is
\begin{align}\label{eq: Path}
\omega = (u_0, x_1, u_1, x_2, u_2,\hdots, x_T, u_T,x_{T+1},\hdots),
\end{align}
where $u_t\in\mathcal{A}$, $x_t\in\mathcal{S}$ for all $t\in\mathbb{Z}_{\geq 0}$ and  $x_t = \delta$ for all $t\geq k$ if and when the system reaches the termination state $\delta\in\mathcal{S}$ in $k$ steps. The objective is to determine the optimal policy $\mu^*$ that minimizes the state value function
\begin{align}\label{eq: Val_func}
J^{\mu}(s) = \mathbb{E}_{p_{\mu}}\Big[\sum_{t=0}^{\infty} \gamma^t c(x_t,u_t,x_{t+1})\big|x_0=s\Big],\quad\forall~s\in\mathcal{S}
\end{align}
where the expectation is with respect to the probability distribution $p_{\mu}(\cdot|s):\omega\rightarrow [0,1]$ on the space of all possible paths $\omega\in\Omega:=\{(u_t,x_{t+1})_{t\in\mathbb{Z}_{\geq 0}}: u_t\in\mathcal{A}, x_t\in\mathcal{S}\}$. In order to ensure that the system reaches the cost-free termination state in finite steps and the optimal state value function $J^{\mu}(s)$ is finite, we make the following assumption throughout this section.
\begin{assumption}\label{assum: assum1}
There exists atleast one deterministic proper policy $\bar{\mu}(a|s)\in\{0,1\}~\forall~a\in\mathcal{A},s\in\mathcal{S}$ such that $\min_{s\in\mathcal{S}}p_{\bar{\mu}}(x_{|\mathcal{S}|}=\delta|x_0=s)>0$. In other words, under the policy $\bar{\mu}$ there is a non-zero probability to reach the cost-free termination state $\delta$, when starting from any state $s$. 
\end{assumption}
We consider the following set of stochastic policies $\mu$
\begin{align}\label{eq: setGamma}
\Gamma:=\{\pi: 0<\pi(a|s)<1~\forall~a\in\mathcal{A},s\in\mathcal{S}\},
\end{align}
and the following lemma ensures that under the Assumption \ref{assum: assum1} all the policies $\mu\in\Gamma$ are {\em proper}; that is, \begin{lemma}\label{lem: lem1}
For any policy $\mu\in\Gamma$ as defined in (\ref{eq: setGamma}), $\min_{s\in\mathcal{S}}p_{\mu}(x_{|\mathcal{S}|}=\delta|x_0=s) > 0$, i.e., under each policy $\mu\in\Gamma$ the probability to reach the termination state $\delta$ in $|\mathcal{S}|=N$ steps beginning from any $s\in\mathcal{S}$, is strictly positive.
\end{lemma}
\noindent{\em Proof. }Please refer to the Appendix \ref{app: AppDerivation}.

We use the Maximum Entropy Principle  to determine the policy $\mu\in\Gamma$ such that the Shannon Entropy of the corresponding distribution $p_{\mu}$ is maximized and the state value function $J^{\mu}(s)$ attains a specified value $J_0$. More specifically, we pose the following optimization problem
\begin{align}\label{eq: OptimP1}
\max_{\{p_{\mu}(\cdot|s)\}:\mu\in\Gamma}\quad H^{\mu}(s) &= -\sum_{\omega\in\Omega}p_{\mu}(\omega|s)\log p_{\mu}(\omega|s)\nonumber\\
\text{subject to}\quad J^{\mu}(s) &= J_0.
\end{align}
{\em Well posedness: }For the class of proper policy $\mu\in\Gamma$ the maximum entropy $H^{\mu}(s)~\forall~s\in\mathcal{S}$ is finite as shown in \cite{biondi2014maximizing, savas2018entropy}. In short, the existence of a cost-free termination state $\delta$ and a non-zero probability to reach it from any state $s\in\mathcal{S}$ ensures that the maximum entropy is finite. Please refer to the Theorem 1 in \cite{biondi2014maximizing} or Proposition 2 in \cite{savas2018entropy} for further details.
\begin{remark}Though the optimization problem in (\ref{eq: OptimP1}) considers the stochastic policies $\mu\in\Gamma$, our algorithms presented in the later sections are designed in such that the resulting stochastic policy asymptotically converges to a deterministic policy.
\end{remark}
\subsection{Problem Solution}\label{sec: ProbSoln}
The probability $p_{\mu}(\omega|s)$ of taking the path $\omega$ in (\ref{eq: Path}) can be determined from the underlying policy $\mu$ by exploiting the Markov property that dissociates $p_{\mu}(\omega|s)$ in terms of the policy $\mu$ and the state transition probability $p$ as
\begin{align}\label{eq: Markov}
\text{$p_{\mu}(\omega|x_0) = \prod_{t=0}^{\infty}\mu(u_t|x_t)p(x_{t+1}|x_t,u_t)$.}
\end{align}
Thus, in our framework we prudently work with the policy $\mu$ which is defined over finite action and state spaces as against the distribution $p_{\mu}(\omega|s)$ defined over infinitely many paths $\omega\in\Omega$. The Lagrangian corresponding to the above optimization problem in (\ref{eq: OptimP1}) is $V^{\mu}_{\beta}(s) = J^{\mu}(s) - \frac{1}{\beta}H^{\mu}(s)=$
\begin{align}\label{eq: Lag}
\mathbb{E}\big[\sum_{t=0}^{\infty}\gamma^tc_{x_tx_{t+1}}^{u_t}+\frac{1}{\beta}(\log \mu_{u_t|x_t}+\log p_{x_tx_{t+1}}^{u_t})\big|x_0=s\big],
\end{align}
where $\beta$ is the Lagrange parameter. Here we have not included the constant value $J_0$ in the cost Lagrangian $V_{\beta}^{\mu}(s)$ for simplicity. We refer to the above Lagrangian $V^{\mu}_{\beta}(s)$ (\ref{eq: Lag}) as the {\em free-energy} function and $\frac{1}{\beta}$ as {\em temperature} due to their close analogies with statistical physics (where free energy is enthalpy (E) minus the temperature times entropy (TH)). To determine the optimal policy $\mu^*_{\beta}$ that minimizes the Lagrangian $V_{\beta}^{\mu}(s)$ in (\ref{eq: Lag}), we first derive the Bellman equation for $V_{\beta}^{\mu}(s)$.
\begin{theorem}\label{thm: HJBNonTrival}
The free-energy function $V_{\beta}^{\mu}(s)$ in (\ref{eq: Lag}) satisfies the following recursive Bellman equation
\begin{align}\label{eq: BellmanTrue}
V_{\beta}^{\mu}(s) = \smashoperator{\sum_{\substack{a,s'\in\mathcal{A},\mathcal{S}}}}\mu_{a|s}&p_{ss'}^{a}\big(\bar{c}_{ss'}^{a}+\frac{\gamma}{\beta}\log\mu_{a|s}+\gamma V_{\beta}^{\mu}(s')+c_0(s)\big),
\end{align}
where $\mu_{a|s}=\mu(a|s)$, $p_{ss'}^a=p(s'|s,a)$,  $\bar{c}_{ss'}^{a}=c(s,a,s')+\frac{\gamma}{\beta}\log p(s'|s,a)$ for simplicity in notation, and function $c_0(s)$ does not depend on the policy $\mu$.
\end{theorem}
\noindent{\em Proof. }Please refer to the Appendix \ref{app: AppDerivation} for details. It must be noted that this derivation shows and exploits the algebraic structure {\small $\sum_{s'}p_{ss'}^aH^{\mu}(s') = \sum_{s'}p_{ss'}^a\log p_{ss'}^a+\log \mu_{a|s}+\lambda_s+1$} as detailed in Lemma \ref{lem: entropyRel} in the appendix.

Now the optimal policy satisfies $\frac{\partial V_{\beta}^{\mu}(s)}{\partial \mu(a|s)}=0$, which results into the Gibbs distribution
\begin{align}
\mu^*_{\beta}(a|s) &= \frac{\exp\big\{-(\beta/\gamma) \Lambda_{\beta}(s,a)\big\}}{\sum_{a'\in\mathcal{A}}\exp\big\{-(\beta/\gamma)\Lambda_{\beta}(s,a')\big\}}, \text{ where }\label{eq: Policy}\\
\Lambda_{\beta}(s,a) &= \sum_{s'\in\mathcal{S}}p_{ss'}^{a}\big(\bar{c}_{ss'}^{a} + \gamma V^{*}_{\beta}(s')\big),\label{eq: Q_bell}
\end{align}
is the state-action value function, 
$p_{ss'}^a = p(s'|s,a)$, $c_{ss'}^a = c(s,a,s')$,  $\bar{c}_{ss'}^a=c_{ss'}^a+\frac{\gamma}{\beta}\log p_{ss'}^a$ and $V^{*}_{\beta}(=V^{\mu^*_{\beta}}_{\beta})$ is the value function corresponding to the policy $\mu^*_{\beta}$. To avoid notional clutter we use the above notations wherever it is clear from the context. Substituting the policy $\mu^*_{\beta}$ in (\ref{eq: Policy}) back into the Bellman equation (\ref{eq: BellmanTrue}) we obtain the {\em implicit} equation
\begin{align}\label{eq: free_energy}
V^{*}_{\beta}(s) = -\frac{\gamma}{\beta}\log\Big(\sum_{a\in\mathcal{A}}\exp\Big\{-\frac{\beta}{\gamma}\Lambda_{\beta}(s,a)\Big\}\Big).
\end{align}
Without loss of generality, we ignore the term $c_0(s)$ in (\ref{eq: BellmanTrue}) since it does not affect the policy as seen from equation (\ref{eq: Policy}). To solve for the state-action value function $\Lambda_{\beta}(s,a)$ and free-energy function $V_{\beta}^*(s)$ we substitute the expression of $V^{*}_{\beta}(s)$ in (\ref{eq: free_energy}) into the expression of $\Lambda_{\beta}(s,a)$ in (\ref{eq: Q_bell}) to obtain the implicit equation $\Lambda_{\beta}(s,a)=:[T\Lambda_{\beta}](s,a)$, where
\begin{align}\label{eq: Q_map}
&[T\Lambda_{\beta}](s,a) = \sum_{s'\in\mathcal{S}}p_{ss'}^{a}\big(c_{ss'}^{a} + \frac{\gamma}{\beta}\log p_{ss'}^{a}\big)\nonumber\\
&\qquad\quad -\frac{\gamma^2}{\beta}\sum_{s'\in\mathcal{S}}p_{ss'}^{a}\log\sum_{a'\in\mathcal{A}}\exp\Big\{-\frac{\beta}{\gamma}\Lambda_{\beta}(s',a')\Big\}.
\end{align}
To solve the above implicit equation, we show that the map $T$ in (\ref{eq: Q_map}) is a contraction map and therefore $\Lambda_{\beta}$ can be obtained using fixed point iterations, which guarantee converging to the unique fixed point. Consequently, the global minimum $V_{\beta}^*$ in (\ref{eq: free_energy}) and the optimal policy $\mu_\beta^*$ in (\ref{eq: Policy}) can be obtained. 
\begin{theorem}\label{thm: ContraMapFinite}
The map $[T\Lambda_{\beta}](s,a)$ as defined in (\ref{eq: Q_map}) is a contraction mapping with respect to a weighted maximum norm, i.e. $\exists$ a vector $\xi=(\xi_{s})\in\mathbb{R}^{|\mathcal{S}|}$ with $\xi_{s} > 0~\forall~s\in\mathcal{S}$ and a scalar $\alpha<1$ such that
\begin{align}
\|T\Lambda_{\beta}-T\Lambda'_{\beta}\|_{\xi}\leq \alpha\|\Lambda_{\beta}-\Lambda'_{\beta}\|_{\xi}
\end{align}
where $\|\Lambda_{\beta}\|_{\xi}=\max\limits_{\substack{s\in\mathcal{S}, a\in\mathcal{A}}}\frac{|\Lambda_{\beta}(s,a)|}{\xi_{s}}$
\end{theorem}
\noindent{\em Proof. }Please refer to Appendix \ref{app: AppCont_MapQ} for details.
\begin{remark}
It is known from the sensitivity analysis \cite{jaynes2003probability} that the value of the Lagrange parameter $\beta$ in (\ref{eq: Lag}) is inversely proportional to the constant $J_0$ in (\ref{eq: OptimP1}). Thus, at small values of $\beta \approx 0$ (equivalently large $J_0$), we are mainly maximizing the Shannon Entropy $H^{\mu}(s)$ and the resultant policy in (\ref{eq: Policy}) encourages exploration along the paths of the MDP. As $\beta$ increases ($J_0$ decreases), more and more weight is given to the state value function $J^{\mu}(s)$ in (\ref{eq: Lag}) and the policy in (\ref{eq: Policy}) goes from being exploratory to being exploitative. As $\beta\rightarrow\infty$ the exploration is completely eliminated and we converge to a deterministic policy $\rightarrow\mu^*$ that minimizes $J^{\mu}(s)$ in (\ref{eq: Val_func}).
\end{remark}
\begin{remark}\label{rem: RemarkER}
We briefly draw  readers' attention to the value function $Y(s) = \mathbb{E}[\sum_{t=0}^{\infty}\gamma^t(c_{x_tx_{t+1}}^{u_t} + (1/\beta)\log \mu_{u_t|x_t})]$ considered in the entropy regularized methods \cite{fox2015taming}. Note that in $Y(s)$ the discounting $\gamma^t$ is multiplied to both the cost term $c_{x_tx_{t+1}}^{u_t}$ as well as the entropy term $(1/\beta)\log\mu_{u_t|x_t}$. However, in our MEP-based method, the Lagrangian $V_{\beta}^{\mu}(s)$ in (\ref{eq: Lag}) comprises of discounting $\gamma^t$ {\em only} over the cost term $c_{x_tx_{t+1}}^{u_t}$ and not on the entropy terms $(1/\beta)\log \mu_{u_t|x_t}$ and $(1/\beta)\log p_{x_tx_{t+1}}^{u_t}$. Therefore, the policy in \cite{fox2015taming} does not satisfy MEP applied over the distribution $p_{\mu}$; consequently their exploration along the paths is not as unbiased as our algorithm.
\end{remark}

\subsection{Model-free Reinforcement Learning Problems}\label{sec: Learning}
In these problems, the cost function $c(s,a,s')$ and the state-transition probability $p(s'|s,a)$ are not known apriori; however, at each discrete time instant $t$ the {\em agent} takes an action $u_t$ under a policy $\mu$ and the {\em environment} (underlying stochastic process) results into an instantaneous cost $c_{x_tx_{t+1}}^{u_t}$ and the subsequently moves to the state $x_{t+1}\thicksim p(\cdot|x_t,u_t)$. Motivated by the iterative updates of Q-learning algorithm \cite{bertsekas1996neuro} we consider the following stochastic updates in our Algorithm \ref{alg: Algorithm1} to learn the state-action value function in our methodology
\begin{align}\label{eq: Q_upd1}
&\text{\small$\Psi_{t+1}(x_t,u_t) = (1-\nu_t(x_t,u_t))\Psi_t(x_t,u_t)$+}\nonumber \\
&\text{\small $\nu_t(x_t,u_t)\Big[c_{x_tx_{t+1}}^{u_t}-\frac{\gamma^2}{\beta}\log\sum_{a'\in\mathcal{A}} \exp\Big\{\frac{-\beta}{\gamma} \Psi_{t}(x_{t+1},a')\Big\}\Big]$},
\end{align}
with the stepsize parameter $\nu_t(x_t,u_t)\in(0,1]$, and show that under appropriate conditions on $\nu_t$ (as illustrated shortly), $\Psi_t$ will converge to the fixed point $\bar{\Lambda}_{\beta}^*$ of the implicit equation \begin{align}\label{eq: Q_map2}
\bar{\Lambda}_{\beta}(s,a) &= \sum_{s'\in\mathcal{S}}p_{ss'}^{a}\Big(c_{ss'}^{a}-\frac{\gamma^2}{\beta}\log\sum_{a'}\exp\big(\frac{-\beta}{\gamma}\bar{\Lambda}_{\beta}(s',a')\big)\Big)\nonumber\\
&=:[\bar{T}\bar{\Lambda}_{\beta}](s,a).
\end{align}
The above equation comprises of a minor change from the equation $\Lambda_{\beta}(s,a)=[T\Lambda_{\beta}](s,a)$ in (\ref{eq: Q_map}). The latter has an additional term $\frac{\gamma}{\beta}\sum_{s'}p_{ss'}^a\log p_{ss'}^a$ which makes it difficult to {\em learn} its fixed point $\Lambda_{\beta}^*$ in the absence of the state transition probability $p_{ss'}^a$ itself. Since in this work we do not attempt to determine (or learn) either the distribution $p_{ss'}^a$ (as in \cite{ross2011bayesian}) from the interactions of the agent with the environment, we work with the approximate state-action value function $\bar{\Lambda}_{\beta}$ in (\ref{eq: Q_map2}) where $\bar{\Lambda}_{\beta}\rightarrow \Lambda_{\beta}$ for large $\beta$ values (since $\frac{\gamma}{\beta}(\sum_{s'}p_{ss'}^a\log p_{ss'}^a)\rightarrow 0$ as $\beta\rightarrow\infty$). The following proposition elucidates the conditions under which the updates $\Psi_t$ in (\ref{eq: Q_upd1}) converge to the fixed point $\bar{\Lambda}_{\beta}^*$.
\begin{proposition}\label{pro: Proposition1}
Consider the class of MDPs illustrated in Section \ref{sec: MDP}. Given that$$\sum_{t=0}^{\infty}\nu_t(s,a)=\infty, \sum_{t=0}^{\infty}\nu_t^2(s,a)<\infty~\forall~s\in\mathcal{S},a\in\mathcal{A},$$
the update $\Psi_t(s,a)$ in (\ref{eq: Q_upd1}) converges to the fixed point $\bar{\Lambda}_{\beta}^*$ of the map $\bar T\bar{\Lambda}_{\beta}\rightarrow \bar{\Lambda}_{\beta}$ in (\ref{eq: Q_map2}) with probability 1. 
\end{proposition}
\noindent{\em Proof. }Please refer to the Appendix \ref{app: Q_learnConv}.
\begin{algorithm}\label{alg: Algorithm1}
\textbf{Input: }$N$, $\nu_t(\cdot,\cdot)$, $\sigma$;  
\textbf{Output: }{$\mu^*$, $\bar{\Lambda}^*$}\\
\textbf{Initialize: }$t=0$, $\Psi_0 = 0$, $\mu_0(a|s) = 1/|\mathcal{A}|$.\\
\For {$episode=1$ to $N$}{
$\beta=\sigma \times epsiode$;
reset environment at state $x_t$\\
\While{True}{
sample $u_t\sim\mu_t(\cdot|x_t)$; obtain cost $c_t$ and $x_{t+1}$\\
update $\Psi_{t}(x_t,u_t)$, $\mu_{t+1}(u_t|x_t)$ in (\ref{eq: Q_upd1}) and (\ref{eq: Policy})\\
break if $x_{t+1}=\delta$; $t \leftarrow t + 1$
}
}
\caption{Model-free Reinforcement Learning}
\end{algorithm}

\begin{remark}
Note that the {\em stochasticity} of the optimal policy $\mu_{\beta}^*(a|s)$ (\ref{eq: Policy}) depends on $\gamma$ value which allows it to incorporate for the effect of the discount factor on its exploration strategy. More precisely, in the case of large discount factors the time window $T$, in which instantaneous costs $\gamma^tc(s_t,a_t,s_{t+1})$ is considerable (i.e., $\gamma^t c_{s_ts_{t+1}}^{a_t}>\epsilon$ $\forall$ $t\leq T$), is large and thus, the stochastic policy (\ref{eq: Policy}) performs higher exploration along the paths.  On the other hand, for small discount factors this time window $T$ is relatively smaller and thus, the stochastic policy (\ref{eq: Policy}) inherently performs lesser exploration. As illustrated in the simulations, this characteristic of the policy in (\ref{eq: Policy}) results into even faster convergence rates in comparison to benchmark algorithms as the discount factor $\gamma$ decreases.
\end{remark}

\section{MDPs with Infinite Shannon Entropy}\label{sec: InfiMDP}
Here we consider the MDPs where the Shannon Entropy $H^{\mu}(s)$ of the distribution $\{p_{\mu}(\omega|s)\}$ over the paths $\omega\in\Omega$ is not necessarily finite (for instance, due to absence of absorption state). To ensure the finiteness of the objective in (\ref{eq: OptimP1}) we consider the {\em discounted} Shannon Entropy \cite{hansen2006robust,zhou2017infinite}
\begin{align}\label{eq: DiscountShannonEnt}
\text{\small$H^{\mu}_d(s) =- \mathbb{E}\Big[\sum_{t=0}^{\infty}\alpha^t(\log \mu_{u_t|x_t}+\log p_{x_tx_{t+1}}^{u_t})\big|x_0=s\Big]$}\end{align}
with a discount factor of $\alpha\in(0,1)$ which we chose to be independent of the discount factor $\gamma$ in the value function $J^{\mu}(s)$. The free-energy function (or, the Lagrangian) resulting from the optimization problem in (\ref{eq: OptimP1}) with the alternate objective function $H^{\mu}_d(s)$ in (\ref{eq: DiscountShannonEnt}) is given by
\begin{align}\label{eq: ValFuncDiscounted}
\text{\small $V_{\beta,I}^{\mu}(s) = \mathbb{E}\Big[\sum_{t=0}^{\infty}\gamma^t\hat{c}_{x_tx_{t+1}}^{u_t} + \frac{\alpha^t}{\beta}\log \mu(u_t|x_t)\big|x_0=s\Big]$},
\end{align}
where $\hat{c}_{x_tx_{t+1}}^{u_t} = c_{x_tx_{t+1}}^{u_t}+\frac{\gamma^t}{\beta\alpha^t}\log p_{x_tx_{t+1}}^{u_t}$, and the subscript $I$ stands for 'infinite entropy' case. Note that the free-energy functions (\ref{eq: Lag}) and (\ref{eq: ValFuncDiscounted}) differ only with regards to the discount factor $\alpha$ and thus, our solution methodology in this section is similar to the one in Section \ref{sec: ProbSoln}.
\begin{theorem}\label{thm: BellmanTrue2}
The free-energy function $V^{\mu}_{\beta,I}(s)$ in (\ref{eq: ValFuncDiscounted}) satisfies the recursive Bellman equation
\begin{align}\label{eq: ValFuncDiscHJB}
\text{\small$V_{\beta,I}^{\mu}(s) = \sum_{\substack{a,s'}}\mu_{a|s}p_{ss'}^a(\check{c}_{ss'}^a + \frac{\gamma}{\alpha\beta}\log\mu_{a|s}+\gamma V_{\beta,I}^{\mu}(s'))$}
\end{align}
where $\check{c}_{ss'}^a=c_{ss'}^a+\frac{\gamma}{\alpha\beta}\log p_{ss'}^a$. 
\end{theorem}
\noindent{\em Proof. }Please see Appendix \ref{app: InfiniteSE}. The above derivation shows and exploits algebraic structure {\small$\alpha\sum_{s'}p_{ss'}^a H^{\mu}_d(s') = \sum_{s'}p_{ss'}^a\log p_{ss'}^a+\log \alpha\mu(a|s) + \lambda_s$} (Lemma \ref{lem: entropyRel2}).

The optimal policy satisfies $\frac{\partial V_{\beta,I}^{\mu}(s)}{\partial \mu(a|s)}=0$, which results into the Gibbs distribution
\begin{align}
\mu^*_{\beta,I}(a|s) &= \frac{\exp\big\{-\frac{\beta\alpha}{\gamma}\Phi_{\beta}(s,a)\big\}}{\sum_{a'\in\mathcal{A}}\exp\big\{-\frac{\beta\alpha}{\gamma}\Phi_{\beta}(s,a')\big\}}, \text{ where}\label{eq: PolicyDiscEnt}\\
\Phi_{\beta}(s,a) &= \sum_{s'\in\mathcal{S}}p_{ss'}^a(\check{c}_{ss'}^a + \gamma V^{*}_{\beta,I}(s')),\label{eq: SADiscEnt}
\end{align}
is the corresponding state-action value function. Substituting the $\mu^*_{\beta,I}$ in (\ref{eq: PolicyDiscEnt}) in the Bellman equation (\ref{eq: ValFuncDiscHJB}) results into the following optimal free-energy function $V_{\beta,I}^*(s)(:=V_{\beta,I}^{\mu^*_{\beta,I}}(s))$.
\begin{align}\label{eq: FEDiscEnt}
V_{\beta,I}^{*}(s) &= -\frac{\gamma}{\alpha\beta}\log\sum_{a'\in\mathcal{A}}\exp\Big(\frac{-\alpha\beta}{\gamma}\Phi_{\beta}(s,a')\Big)
\end{align}
\begin{remark}
The subsequent steps to learn the optimal policy $\mu^*_{\beta,I}$ in (\ref{eq: PolicyDiscEnt}) are similar to the steps demonstrated in the Section \ref{sec: Learning}. We forego the similar analysis here.
\end{remark}
\begin{remark}
When $\alpha=\gamma$ the policy $\mu^*_{\beta,I}$ in (\ref{eq: PolicyDiscEnt}), state-action value function $\Phi_{\beta}$ in (\ref{eq: SADiscEnt}) and the free-energy function $V^*_{\beta,I}$ in (\ref{eq: FEDiscEnt}) corresponds to the similar expressions that are obtained in the Entropy regularized methods \cite{fox2015taming}. However, in this paper we do not require that $\alpha = \gamma$. On the other hand, we propose that $\alpha$ should take up large values. In fact, our simulations in the Section \ref{sec: Simulations} demonstrate better convergence rates that are obtained when $\gamma< \alpha=(1-\epsilon)$ as compared to when $\gamma=\alpha$.
\end{remark}

\section{Parameterized MDPs}\label{sec: ParMDP}
\subsection{Problem Formulation}\label{sec: ParMDPProbForm}
As stated in Section \ref{sec: Intro}, many application areas such as small cell networks (Figure \ref{fig: 5G_smallCell}) pose a parmaterized MDP that requires simultaneously determining the (a) optimal policy $\mu^*$, and (b) the unknown state and action parameters $\zeta=\{\zeta_s\} $ and $\eta=\{\eta_a\}$ such that the state value function 
\begin{align}\label{eq: ParValFunc}
J^{\mu}_{\zeta\eta}(s) = \mathbb{E}_{p_{\mu}}\Big[\sum_{t=0}^{\infty}\gamma^t c\big(x_t(\zeta),u_t(\eta),x_{t+1}(\zeta)\big)\Big|x_0=s\Big]
\end{align}
is minimized $\forall$ $s\in\mathcal{S}$ where $x_t(\zeta)$ denotes the state $x_t\in\mathcal{S}$ with the associated parameter $\zeta_{x_t}$ and $u_t(\eta)$ denotes the action $u_t\in\mathcal{A}$ with the associated action parameter value $\eta_{u_t}$. As in Section \ref{sec: MDP}, we assume that the parameterized MDPs exhibit atleast one deterministic proper policy (Assumption \ref{assum: assum1}) to ensure the finiteness of the value function $J^{\mu}_{\zeta\eta}(s)$ and the Shannon Entropy $H^{\mu}(s)$ of the MDP for all $\mu\in\Gamma$. We further assume that the state-transition probability $\{p_{ss'}^a\}$ is independent of the state and action parameters $\zeta,\eta$.
\subsection{Problem Solution}\label{sec: ParaProbSoln}
This problem was solved in Section \ref{sec: ProbSoln}, where the states and actions were not parameterized, or equivalently can be viewed as if parameters $\zeta$, $\eta$ were known and fixed. For the parameterized case, we apply the same solution methodology, which results in the same optimal policy $\mu^*_{\beta,\zeta\eta}$ as in (\ref{eq: Policy}) as well as the corresponding free-energy function $V_{\beta,\zeta\eta}^*(s)$ in (\ref{eq: free_energy}) except that now they are characterized by the parameters $\zeta$, $\eta$. To determine the optimal (local) parameters $\zeta$, $\eta$, we set {\small $\sum_{s'\in\mathcal{S}}\frac{\partial V_{\beta,\zeta\eta}^{*}(s')}{\partial \zeta_s}=0$} $\forall$ $s$, and {\small $\sum_{s'\in\mathcal{S}}\frac{\partial V_{\beta,\zeta\eta}^{*}(s')}{\partial\eta_a}=0$} $\forall$ $a$, which we implement by using the gradient descent steps
\begin{align}\label{eq: gard_desc}
\zeta_s^+ = \zeta_s - \eta\sum_{s'\in\mathcal{S}}G_{\zeta_s}^{\beta}(s'),~
\eta_a^+ = \eta_a - \bar{\eta}\sum_{s'\in\mathcal{S}}G_{\eta_a}^{\beta}(s').
\end{align}
Here {\small$G_{\zeta_s}^{\beta}(s'):=\frac{\partial V_{\beta,\zeta\eta}^*(s')}{\partial \zeta_s}$} and {\small$G_{\eta_a}^{\beta}(s'):=\frac{\partial V_{\beta,\zeta\eta}^*(s')}{\partial \eta_a}$}. The derivatives $G_{\zeta_s}^{\beta}$ and $G_{\eta_a}^{\beta}$ are assumed to be bounded (see Proposition \ref{pro: Proposition2}). We compute these derivatives as {\small$G_{\zeta_s}^{\beta}(s')=\sum_{a'}\mu_{a'|s'}K_{\zeta_s}^{\beta}(s',a')$} and {\small$G_{\eta_a}^{\beta}(s')=\sum_{a'}\mu_{a'|s'}L_{\eta_a}^{\beta}(s',a')$} $\forall$ $s'\in\mathcal{S}$ where $K_{\zeta_s}^{\beta}(s',a')$ and $L_{\eta_a}^{\beta}(s',a')$ are the fixed points of their corresponding Bellman equations $K_{\zeta_s}^{\beta}(s',a')=[T_1K_{\zeta_s}^{\beta}](s,a)$ and $L_{\eta_a}^{\beta}(s',a')=[T_2L_{\eta_a}^{\beta}](s',a')$ where 
\begin{align}\label{eq: V_zeta_j}
\begin{split}
[T_1K_{\zeta_s}^{\beta}](s',a')&=\sum_{\substack{s''}}p_{s's''}^{a'}\Big[\frac{\partial c_{s's''}^{a'}}{\partial \zeta_s} + \gamma G_{\zeta_s}^{\beta}(s'')\Big],\\
[T_2L_{\eta_a}^{\beta}](s',a')&=\sum_{\substack{s''}}p_{s's''}^{a'}\Big[\frac{\partial c_{s's''}^{a'}}{\partial \eta_a} + \gamma G_{\eta_a}^{\beta}(s'')\Big].
\end{split}
\end{align}
Note that in the above equations we have suppressed the dependence of the instantaneous cost function $c_{s's''}^{a'}$ on the parameters $\zeta$ and $\eta$ to avoid notational clutter.
\begin{theorem}\label{thm: ParaDerivatives}
The operators $[T_1K_{\zeta_s}^{\beta}](s',a')$ and $[T_2L_{\eta_a}^{\beta}](s',a')$ defined in (\ref{eq: V_zeta_j}) are contraction maps with respect to a weighted maximum norm $\|\cdot\|_{\xi}$, where $\|X\|_{\xi}=\max_{s',a'}\frac{X(s',a')}{\xi_{s'}}$ and $\xi\in\mathbb{R}^{|\mathcal{S}|}$ is a vector of positive components $\xi_s$.
\end{theorem}
\noindent{\em Proof. }Please refer the Appendix \ref{app: Q_learnConv} for details.

As previously stated in Section \ref{sec: Intro}, the state value function $J^{\mu}_{\zeta\eta}(\cdot)$ in (\ref{eq: ParValFunc}) is generally non-convex function of the parameters $\zeta$, $\eta$ and riddled with multiple poor local minima with the resulting optimization problem being possibly NP-hard \cite{mahajan2009planar}. In our algorithm for parameterized MDPs we anneal $\beta$ from $\beta_{\min}$ to $\beta_{\max}$, similar to our approach for non-parameterized MDPs in Section \ref{sec: ProbSoln}, where the solution from the current $\beta$ iteration is used to initialize the subsequent $\beta$ iteration. However, in addition to facilitating a steady transition from an exploratory policy to an exploitative policy, annealing facilitates a gradual homotopy from the convex negative Shannon entropy function to the non-convex state value function $J^{\mu}_{\zeta\eta}$ which prevents our algorithm from getting stuck in a poor local minimum. The underlying idea of our heuristic is to track the optimal as the initial convex function deforms to the actual non-convex cost. Also, minimizing the Lagrangian $V_{\beta}^*(s)$ at $\beta=\beta_{\min}\approx 0$ determines the global minimum thereby making our algorithm insensitive to initialization. Algorithm \ref{alg: Algorithm2} illustrates  steps to determine policy and parameters for a parameterized MDP.
\begin{algorithm}\label{alg: Algorithm2}
\textbf{Input: } $\beta_{\min}$, $\beta_{\max}$, $\tau$;  \textbf{Output: } $\mu^*$, $\zeta$ and $\eta$.\\
\textbf{Initialize:} $\beta = \beta_{\min}$, $\mu_{a|s} = \frac{1}{|\mathcal{A}|}$, and  $\zeta,\eta$ to $0$\\
\While{$\beta\leq \beta_{\max}$}{
\While{True}{
\While{until convergence}{
update $\Lambda_{\beta}$,$\mu_{\beta}$,$G_{\zeta_s}^{\beta}$,$G_{\eta_a}^{\beta}$ in (\ref{eq: Q_bell}), (\ref{eq: Policy}) and (\ref{eq: V_zeta_j})}
update $\zeta,\eta$ in (\ref{eq: gard_desc})
\textbf{if} $\|G_{\zeta_s}\|,\|G_{\eta_a}\|<\epsilon$, break
}
$\beta \leftarrow \tau\beta$\\
}
\caption{Parameterized Markov Decision Process}
\end{algorithm}

\subsection{Parameterized Reinforcement Learning}\label{sec: ParRL}
In many applications, formulated as parameterized MDPs, the explicit knowledge of the cost function $c_{ss'}^a$, its dependence on the parameters $\zeta$, $\eta$, and the state-transition probabilities $\{p_{ss'}^a\}$  are not known. However, for each action $a$ the environment results into an instantaneous cost based on its current $x_t$, next state $x_{t+1}$ and the parameter $\zeta$, $\eta$ values which can subsequently be used to simultaneously learn the policy $\mu^*_{\beta,\zeta\eta}$ and the unknown state and action parameters $\zeta$, $\eta$ via stochastic iterative updates. At each $\beta$ iteration in our learning algorithm, we employ the stochastic iterative updates in (\ref{eq: Q_upd1}) to determine the optimal policy $\mu^*_{\beta,\zeta\eta}$ for given $\zeta$, $\eta$ values and subsequently employ the stochastic iterative updates
\begin{align}\label{eq: derivative_updates}
\text{\small $K_{\zeta_s}^{t+1}(x_t,u_t)$}&\text{\small$=(1-\nu_t(x_t,u_t))K_{\zeta_s}^t(x_t,u_t)+$}\nonumber\\
&\quad \text{\small $\nu_t(x_t,u_t)\Big[\frac{\partial c_{x_tx_{t+1}}^{u_t}}{\partial \zeta_s}+\gamma G_{\zeta_s}^t(x_{t+1})\Big],$}
\end{align}
where $G_{\zeta_s}^t(x_{t+1})=\sum_a\mu_{a|x_{t+1}}K_{\zeta_s}^t(x_{t+1},a)$ to learn the derivative $G_{\zeta_s}^{\beta*}(\cdot)$. Similar updates are used to learn $G_{\eta_a}^{\beta*}(\cdot)$. The parameter values $\zeta$, $\eta$ are then updated using the gradient descent step in (\ref{eq: gard_desc}). The following proposition formalizes the convergence of the updates in (\ref{eq: derivative_updates}) to the fixed point $G_{\zeta_s}^{\beta*}$.
\begin{proposition}\label{pro: Proposition2}
For the class of parameterized MDPs considered in Section \ref{sec: ParMDPProbForm} given that
\begin{enumerate}
\item$\sum_{t=0}^{\infty}\nu_t(s,a)=\infty$, $\sum_{t=0}^{\infty}\nu_t^2(s,a)<\infty$ $\forall s\in\mathcal{S}$, $a\in\mathcal{A}$,
\item$\exists$ $B>0$ such that $\Big|\frac{\partial c(s',a',s'')}{\partial \zeta_s}\Big|\leq B$ $\forall~s,s',a',s''$,
\item$\exists$ $C>0$ such that $\Big|\frac{\partial c(s',a',s'')}{\partial \eta_a}\Big|\leq C$ $\forall~a,s',a',s''$,
\end{enumerate}
the updates in (\ref{eq: derivative_updates}) converge to the unique fixed point $G_{\zeta_s}^{\beta*}(s')$ of the map $T_1:G_{\zeta_s}\rightarrow G_{\zeta_s}$ in (\ref{eq: V_zeta_j}).
\end{proposition}
\noindent{\em Proof. }Please refer to the Appendix \ref{app: Q_learnConv} for details.
\begin{algorithm}\label{alg: Algorithm3}
\textbf{Input: } $\beta_{\min}$, $\beta_{\max}$, $\tau$, $T$, $\nu_t$;  \textbf{Output: } $\mu^*$, $\zeta$, $\eta$\\
\textbf{Initialize: }$\beta = \beta_{\min}$, $\mu_{t}=\frac{1}{|\mathcal{A}|}$, and $\zeta,\eta,G_{\zeta}^{\beta},G_{\eta}^{\beta}$, $K_{\zeta}^{\beta},L_{\eta}^{\beta}$, $\bar{\Lambda}_{\beta}$ to $0$.\\
\While{$\beta\leq \beta_{\max}$}{Use Algorithm \ref{alg: Algorithm1} to obtain $\mu^*_{\beta,\zeta\eta}$ at given $\zeta$, $\eta$, $\beta$.\\
    Consider {\em env1}($\zeta$,$\eta$), {\em env2}($\zeta',\eta'$); set $\zeta'=\zeta$, $\eta'=\eta$\\
    \While{$\{\zeta_s\}$, $\{\eta_a\}$ converge}{
    \For{$\forall s\in\mathcal{S}$}{
        \For{$episode = 1$ to $T$}{
            reset {\em env1, env2} at state $x_t$, \\
            \While{True}{
                sample action $u_t\sim\mu^*(\cdot|x_t)$.\\
                \textit{env1:} obtain $c_t$, $x_{t+1}$.\\
                \textit{env2:} set {\small$\textstyle\zeta_s'=\zeta_s+\Delta\zeta_s$}, get $c_t'$, $x_{t+1}$.\\
                find {\small$\textstyle G_{\zeta_s}^{t+1}(x_{t})$} with $\scriptstyle\frac{\partial c_{x_tx_{t+1}}^{u_t}}{\partial \zeta_s}\approx\frac{c'_t-c_t}{\Delta\zeta_s}$.\\
                break if $x_{t+1}=\delta$; $t\leftarrow t+1$.
            }
        }
    }
    Similarly learn $G_{\eta_a}^{\beta}$. Update $\{\zeta_s\}$, $\{\eta_a\}$ in (\ref{eq: gard_desc}).}
    $\beta \leftarrow \tau\beta$\\
}
\caption{Parameterized Reinforcement Learning}
\end{algorithm}

{\em Algorithmic Details: } Please refer to the Algorithm \ref{alg: Algorithm3} for a complete implementation. Unlike the scenario in Section \ref{sec: Learning} where the agent acts upon the environment by taking an action $u_t\in\mathcal{A}$ and learns {\em only} the policy $\mu^*$, here the agent interacts with the environment by (a) taking an action $u_t\in\mathcal{A}$ and {\em also} providing (b) estimated parameter $\zeta$, $\eta$ values to the environment; subsequently, the environment results into an instantaneous cost and the next state. In our Algorithm \ref{alg: Algorithm3}, we first learn the policy $\mu_{\beta}^*$ at a given value of the parameters $\zeta$, $\eta$ using the iterations (\ref{eq: Q_upd1}) and then learn the fixed points $G_{\zeta_s}^{\beta*}$, $G_{\zeta_a}^{\beta*}$ using the iterations in (\ref{eq: derivative_updates}) to update the parameters $\zeta$, $\eta$ using (\ref{eq: gard_desc}). Note that the iterations (\ref{eq: derivative_updates}) require the derivatives $\partial c(s',a',s'')/\partial \zeta_s$ and $\partial c(s',a',s'')/\partial \eta_a$ which we determine using the instantaneous costs resulting from two {\em $\epsilon$-distinct environments} and finite difference method. Here the $\epsilon$-distinct environments represent the same underlying MDP but are distinct only in one of the parameter values. However, if two $\epsilon$-distinct environments are not feasible one can work with a single environment where the algorithm stores the instantaneous costs and the corresponding parameter values upon each interaction with the environment. 
\begin{remark}
Parameterized MDPs with infinite Shannon entropy $H^{\mu}$ can be analogously addressed using above methodology. 
\end{remark}
\begin{remark}\label{rem: Remark8}
The MDPs addressed in Section \ref{sec: gen_inst}, \ref{sec: InfiMDP}, and \ref{sec: ParMDP} consider different variants of the discounted infinite horizon problems. MDPs in Section \ref{sec: gen_inst} address the class of sequential problems that have a non-zero probability of reaching a cost-free termination state (i.e., a finite Shannon entropy value). MDPs considered in Section \ref{sec: InfiMDP} need not reach a termination state (possibly infinite value of Shannon entropy), and the underlying sequential decision problem continues for the length of horizon determined by the discounting factor $\gamma$. Parameterized MDPs in Section \ref{sec: ParMDP} can have finite or infinite Shannon entropy, but they comprise of states and actions that have an unknown parameter associated to them. 
\end{remark}

\section{Simulations}\label{sec: Simulations}
We broadly classify our simulations into two categories. Firstly, in the model-free RL setting we demonstrate our Algorithm \ref{alg: Algorithm1} to determine the control policy \begin{wrapfigure}{r}{0.41\columnwidth}
    \centering
    \includegraphics[width=0.41\columnwidth]{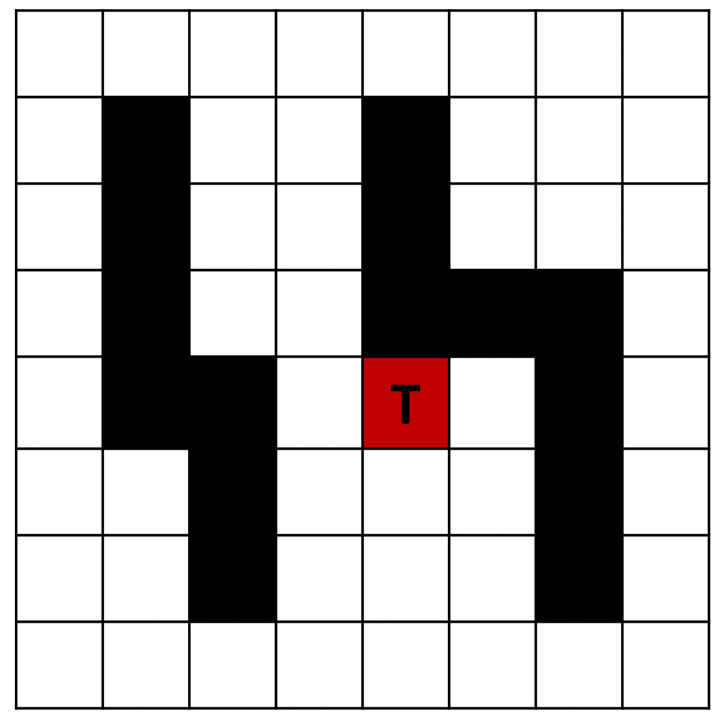}\vspace{-0.3cm}
    \caption{Gridworld environment.}
    \label{fig: DC}\vspace{-0.3cm}
\end{wrapfigure}  $\mu^*$ for the finite and infinite Shannon entropy variants of the Gridworld environment in Figure \ref{fig: DC}. Each cell in the Gridworld denotes a state. The cells colored black are invalid states. An agent can choose to move vertically, horizontal, diagonally or stay at the current cell. Each action is followed by a probability to slip in the neighbouring states (probability of 0.05 to slip in each of the vertical and horizontal directions, and probability of 0.025 to slip in each of the diagonal directions - cumulative $p(slip)\approx 0.3$). For the finite entropy case - each step incurs a unit cost. The process terminates when the agent reaches the terminal state $\mathbf{T}$. For the infinite entropy case - each step incurs a unit reward. Secondly, in the parameterized MDPs and RL setting we demonstrate our Algorithms \ref{alg: Algorithm2} and \ref{alg: Algorithm3} in designing a 5G small cell network. This involves simultaneously determining the locations of the small cells in the network as well as the optimal routing path of the communication packets from the base station to the users.

We compare our MEP-based Algorithm \ref{alg: Algorithm1} with the benchmark algorithms Entropy Regularized (ER) G-learning (also referred to as Soft Q-learning) \cite{fox2015taming}, Q-learning \cite{watkins1992q} and Double Q-learning \cite{hasselt2010double}. Note that our choice of the benchmark algorithm G-learning (or, entropy regularized Soft Q) presented in \cite{fox2015taming}) is based on its commonality to our framework as discussed in the Section \ref{sec: ProbSoln}, and the choice of algorithms Q-learning and Double Q-learning is based on their widespread utility in the literature. Also, note that the work done in \cite{fox2015taming} already establishes the efficacy of the G-learning algorithm over the following algorithms in literature Q-learning, Double Q-learning, $\Psi$-learning \cite{rawlik2010approximate}, Speedy Q-learning \cite{ghavamzadeh2011speedy}, and the consistent Bellman Operator $\mathcal{T}_C$ of \cite{bellemare2016increasing}. Below we highlight features and advantages of our MEP-based Algorithm \ref{alg: Algorithm1}.

\begin{figure*}
\centering
\includegraphics[width=0.95\textwidth]{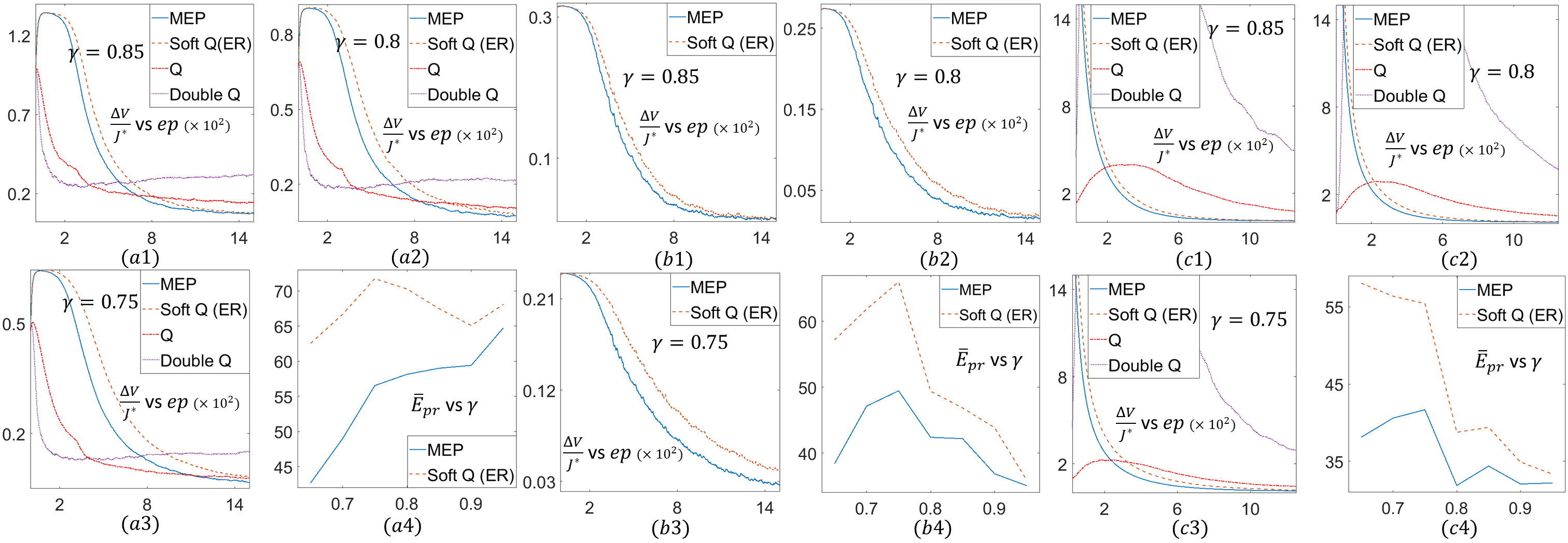}\vspace{-0.3cm}
\caption{Performance of MEP-based algorithm: Illustrations on Gridworld Environment in Figure \ref{fig: DC}. (a1)-(a3) Finite Entropy variant:  Illustrates faster convergence of Algorithm \ref{alg: Algorithm1} (MEP) at different $\gamma$ values. (a4) Demonstrates faster rates of convergence of Algorithm \ref{alg: Algorithm1} (MEP) for $\gamma$ values ranging from $0.65$ to $0.95$. (b1)-(b3) {\em Infinite Entropy Variant} : Demonstrates faster convergence of Algorithm \ref{alg: Algorithm1} (MEP) to $J^*$. (b4) Illustrates the consistent faster convergence rates of MEP with $\gamma$ ranging from $0.65$ to $0.95$. (c1)-(c4) Finite Entropy Version (with added Gaussian noise) : Similar observations as in (a1)-(a4) with significantly higher instability in learning with Double Q algorithm.}
\label{fig: DC_Simulation}\vspace{-0.4cm}
\end{figure*}

{\em Faster Convergence to Optimal $J^*$: } Figures \ref{fig: DC_Simulation}(a1)-(a3) (finite entropy variant of Gridworld) and Figures \ref{fig: DC_Simulation} (b1)-(b3) (infinite entropy variant of Gridworld) illustrate the faster convergence of our MEP-based Algorithm \ref{alg: Algorithm1} for different discount factor $\gamma$ values. Here, at each episode the percentage error $\Delta V/J^*$ between the value function $V_{\beta}^{\mu}$ corresponding to {\em learned} policy $\mu=\mu(ep)$ in the episode $ep$, and the optimal value function $J^*$ is given by
\begin{align}
\text{\small$\frac{\Delta V(ep)}{J^*} = \frac{1}{N}\sum_{i=1}^N\sum_{s\in\mathcal{S}}\frac{\big|V_{\beta,i}^{\mu(ep)}(s)-J^*(s)\big|}{J^*(s)}$},
\end{align}
where $N$ denotes the total experimental runs and $i$ indexes the value function $V_{\beta,i}^{\mu}$ for each run. As observed in Figures \ref{fig: DC_Simulation}(a1)-(a3), and Figures \ref{fig: DC_Simulation}(b1)-(b3), our Algorithm \ref{alg: Algorithm1} converges even faster as the discount factor $\gamma$ decreases. We characterize the faster convergence rates also in terms of the convergence time - more precisely the  {\em percentage} $\bar{E}_{pr}$ of total episodes taken for the learning error $\Delta V/J^*$ to reach within $5\%$ of the best (see  Figures \ref{fig: DC_Simulation}(a4) and \ref{fig: DC_Simulation}(b4)). As is observed in the figures, the performance of our (MEP-based) algorithm in comparison to entropy regularized G learning is better across all values ($0.65$ to $0.95$) of discount factor $\gamma$. Note that the performance of Algorithm \ref{alg: Algorithm1} gets even better with decreasing $\gamma$ values where the smaller discount factor values occur in instances such as the context of recommendation systems \cite{zheng2018drn}, and teaching RL-agents using human-generated rewards \cite{knox2012reinforcement}.

{\em Robustness to noise in data: }Figures \ref{fig: DC_Simulation}(c1)-(c4) demonstrate robustness to noisy environments; here the instantaneous cost $c(s,a,s')$ in the finite horizon variant of Gridworld is noisy. For the purpose of simulations, we add Gaussian noise $\mathcal{N}(0,\sigma^2)$ with $\sigma=1$ for vertical and horizontal actions, and $\sigma=0.5$ for diagonal movements. Here, at each episode we compare the percentage error $\Delta V/J^*$ in the {\em learned} value functions $V_{\beta}$ (corresponding to the state-action value estimate in (\ref{eq: Q_upd1})) of the respective algorithms. Similar to our observations and conclusions in Figures \ref{fig: DC_Simulation}(a1)-(a3) and Figures \ref{fig: DC_Simulation}(b1)-(b3) we see faster convergence of our MEP-based algorithm over the benchmark algorithms in Figures \ref{fig: DC_Simulation}(c1)-(c3) in the case of noisy environment. Also, Figure \ref{fig: DC_Simulation}(c4) demonstrates that across all discount factor values ($0.65$ to $0.95$), Algorithm \ref{alg: Algorithm1} converges faster than the entropy regularized Soft Q learning.

{\em Simultaneously determining the unknown parameters and policy in Parameterized MDPs: } We design the 5G Small Cell Network (see Figure \ref{fig: 5G_smallCell}) both when the underlying model ($c_{ss'}^a$ and $p_{ss'}^a$) is known (using Algorithm \ref{alg: Algorithm2}) and as well as unknown (using Algorithm \ref{alg: Algorithm3}). In our simulations we randomly distribute $46$ user nodes $\{n_i\}$ at $\{x_i\}$ and the base station $\delta$ at $z$ in the domain $\Omega\subset\mathbb{R}^2$ as shown in Figure \ref{fig: 5G_SC}(a). The objective is to determine the locations $\{y_j\}_{j=1}^5$ (parameters) of the small cells $\{f_j\}_{j=1}^5$ and determine the corresponding communication routes (policy). Here, the state space of the underlying MDP is $\mathcal{S}=\{n_1,\hdots,n_{46},f_1,\hdots,f_{5}\}$ where the locations $y_1,\hdots,y_5$ of the small cells are the unknown parameters $\{\zeta_s\}$ of the MDP, the action space is $\mathcal{A}=\{f_1,\hdots,f_5\}$, and the cost function $c(s,a,s')=\|\rho(s)-\rho(s')\|_2^2$ where $\rho(\cdot)$ denotes the spatial location of the respective states. The objective is to simultaneously determine the parameters (unknown small cell locations) and the control policy (communication routes in the 5G network). We consider two cases where (a) $p_{ss'}^a$ is deterministic, i.e. an action $a$ at the state $s$ results into $s'=a$ with probability $1$, and (b) $p_{ss'}^a$ is probabilistic such that action $a$ at the state $s$ results into $s'=a$ with probability $0.9$ or to the state $s'=f_1$ with probability $0.1$. Additionally, due to absence of prior work in literature on network design problems modeled as parameterized MDPs, we compare our results only with the solution resulting from a straightforward sequential methodology (as shown in Figure \ref{fig: 5G_SC}(a)) where we first partition the user nodes into $5$ distinct clusters to allocate a small cells in each cluster, and then determine optimal routes in the network.

\begin{figure*}
\centering
\includegraphics[width=0.93\textwidth]{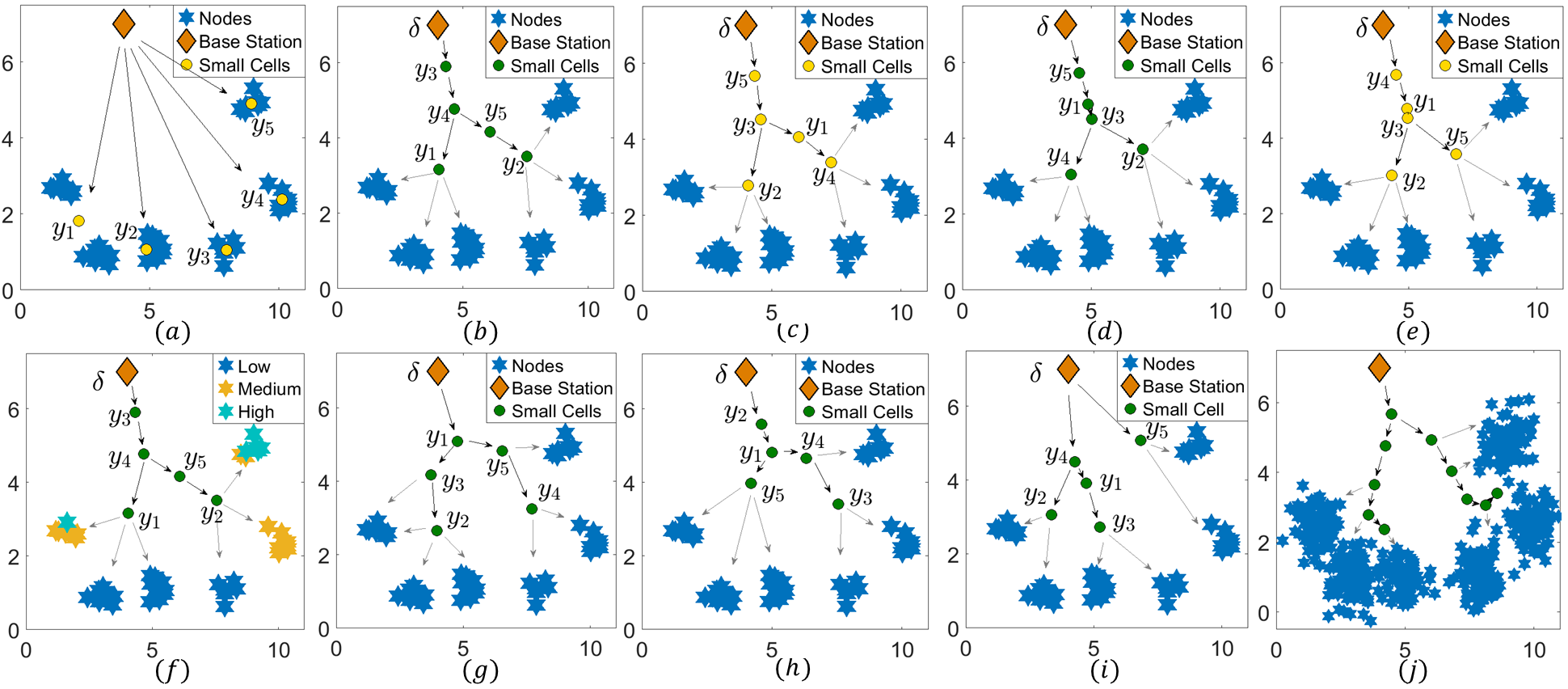}\vspace{-0.3cm}
\caption{Parameterized MDPs and RL - Design of 5G small cell network. State space $\mathcal{S}=\{\{n_i\},\{f_j\},\delta\}$ comprises of the user nodes $\{n_i\}$, small cells $\{f_j\}$, and base station $\delta$. The unknown parameters $\zeta_s$ denote the locations $\{y_j\}$ of the small cells. Action space comprises of the small-cells $\mathcal{A}=\{f_j\}$. Based on our modelling of the network there are no unknown action parameters $\{\eta_a\}$/
(a) Illustrates small cell locations $\{y_j\}$ and communication routes determined using a straightforward sequential methodology. (b)-(c) demonstrate small cells at $\{y_j\}$ and communication routes (as illustrated by arrows) resulting from policy obtained from Algorithm \ref{alg: Algorithm2} (model-based) and Algorithm \ref{alg: Algorithm3} (model-free), respectively when the $p_{ss'}^a$ is deterministic. (d)-(e) solutions obtained using Algorithm \ref{alg: Algorithm2} and Algorithm \ref{alg: Algorithm3}, respectively when $p_{ss'}^a$ is probabilistic. (f) sensitivity analysis of the solutions with respect to user node locations $\{x_i\}$. (g)-(h) Network design obtained when considering entropy of the distribution over the control actions and paths of the MDP, respectively. (i) Network design obtained without annealing in Algorithm \ref{alg: Algorithm2}. (j) Simulation on a larger dataset (user base increased by more than $10$ times).}
\label{fig: 5G_SC}\vspace{-0.4cm}
\end{figure*}

{\em Deterministic $p(s,a,s')$: }Figure \ref{fig: 5G_SC}(b) illustrates the allocation of small cells and the corresponding communication routes (resulting from optimal policy $\mu^*$) as determined by the Algorithm \ref{alg: Algorithm2}. Here, the network is designed to minimize the cumulative cost of communication from each user node and small cell. As denoted in the Figure, the route $\delta\rightarrow y_3\rightarrow y_4\rightarrow y_1\rightarrow n_i$ carries the communication packet from the base station $\delta$ to the respective user nodes $n_i$ as indicated by the grey arrow from $y_1$. The cost incurred here is approximately $180\%$ lesser than that in Figure \ref{fig: 5G_SC}(a) clearly indicating the advantage obtained from simultaneously determining the parameters and policy over a sequential methodology. In the model-free RL setting where the functions $c(s,a)$, $p_{ss'}^a$, and the locations $\{x_i\}$ of the user nodes $\{n_i\}$ are not known, we employ our Algorithm \ref{alg: Algorithm3} to determine the small cell locations $\{y_j\}_{j=1}^5$ and as well as the optimal policy $\{\mu^*(a|s)\}$ as demonstrated in the Figure \ref{fig: 5G_SC}(c). It is evident from Figures \ref{fig: 5G_SC}(b) and (c) that the solutions obtained when the model is completely known and unknown are approximately same. In fact, the solutions obtained differ only by $1.9\%$ in terms of the total cost $\sum_{s\in\mathcal{S}}J^{\mu}_{\zeta\eta}(s)$ (\ref{eq: ParValFunc}) incurred, clearly indicating the efficacy of our model-free learning Algorithm \ref{alg: Algorithm3}. 

{\em Probabilistic $p(s,a,s')$: }Figure \ref{fig: 5G_SC}(d) illustrates the solution as obtained by our Algorithm \ref{alg: Algorithm2} when the underlying model ($c(s,a)$, $p_{ss'}^a$, and $\{x_i\}$) is known. As before, here the network is designed to minimize the cumulative cost of communication from each user node and small cell. The cost associated to the network design is approximately $127\%$ lesser than in Figure \ref{fig: 5G_SC}(a). Figure \ref{fig: 5G_SC}(e) illustrates the solution as obtained by the Algorithm \ref{alg: Algorithm3} for the model-free case ($c(s,a)$, $p_{ss'}^a$, and $\{x_i\}$ are unknown). Similar to the above scenario, the solutions obtained for this case using the Algorithms \ref{alg: Algorithm2} and \ref{alg: Algorithm3} are also approximately the same and differ only by $0.3\%$ in terms of the total cost $\sum_{s}J^{\mu}_{\zeta\eta}(s)$ incurred; thereby, substantiating the efficacy of our proposed model-free learning Algorithm \ref{alg: Algorithm3}.

{\em Sensitivity Analysis: } Our algorithms enable categorizing the user nodes $\{n_i\}$ in Figure \ref{fig: 5G_SC}(b) into the categories of (i) low, (ii) medium, and (iii) high sensitiveness such that the final solution is least susceptible to the user nodes in (i) and most susceptible to the nodes in (iii). Note that the above sensitivity analysis requires to compute the derivative $\sum_{s'}\partial V_{\beta}^{\mu}(s')/\partial \zeta_{s}$, and we determine it by solving for the fixed point of the Bellman equation in (\ref{eq: V_zeta_j}). The derivative $\sum_{s'}{\partial V_{\beta}^{\mu}(s')}/{\partial \zeta_{s}}$ computed at $\beta\rightarrow\infty$ is a measure of sensitivity of the solution to the cost function $\sum_s J^{\mu}_{\zeta\eta}(s)$ in (\ref{eq: ParValFunc}) since $V_{\beta}^{\mu}$ in (\ref{eq: free_energy}) is a smooth approximation of $J^{\mu}_{\zeta\eta}(s)$ in (\ref{eq: ParValFunc}) and $V_{\beta}^{\mu}\rightarrow J^{\mu}_{\zeta\eta}(s)$ as $\beta\rightarrow \infty$. A similar analysis for Figure \ref{fig: 5G_SC}(c)-(e) can be done if the locations $\{x_i\}$ of the user nodes $\{n_i\}$ are known to the {\em agent}. The sensitivity of the final solution to the locations $\{y_j\}$, $z$ of the small cells and the base station can also be determined in a similar manner.

{\em Entropy over paths versus entropy of the policy: }  We demonstrate the benefit of maximizing the entropy of the distribution $\{p_{\mu}(\omega|s)\}$ over the paths of an MDP as compared to the distribution $\{\mu(a|s)\}$ over the control actions. Figure \ref{fig: 5G_SC}(g) demonstrates the 5G network obtained by considering the distribution over the control policy, and the Figure \ref{fig: 5G_SC}(h) illustrates the network obtained by considering the distribution over the entire paths. The network cost incurred in the Figure \ref{fig: 5G_SC}(h) is $5\%$ less than the cost incurred in Figure \ref{fig: 5G_SC}(g). Here, we have considered the above demonstrated probabilistic $p_{ss'}^a$ scenario and minimized the cumulative communication cost incurred only from the user nodes.

{\em Avoiding poor local minima and large scale setups: } As noted in the Section \ref{sec: ParaProbSoln}, annealing $\beta$ from a small value $\beta_{\min}(\approx 0)$ to a large value $\beta_{\max}$ prevents the algorithm from getting stuck at a poor local minima. Figure \ref{fig: 5G_SC}(i) demonstrates the network design obtained where the Algorithm \ref{alg: Algorithm2} does not anneal $\beta$, and iteratively solves the optimization problem at $\beta=\beta_{\max}$. The resulting network incurs a $11\%$ higher cost in comparison to the network obtained in \ref{fig: 5G_SC}(h) where the Algorithm \ref{alg: Algorithm2} anneals $\beta$ from a small to a large value. Figure \ref{fig: 5G_SC}(j) demonstrate the 5G network design obtained using Algorithm \ref{alg: Algorithm2} when the user nodes are increased by around $12$ times ($610$), and the allocated small cells are doubled to $10$.

\section{Analysis and Discussion}\label{sec: Ana_Disc}
\subsubsection{Mutual Information Minimization} The optimization problem (\ref{eq: OptimP1}) maximizes the Shannon entropy $H^{\mu}(s)$ under a given constraint on the value function $J^{\mu}$. We can similarly pose and solve the mutual information minimization problem that requires to determine the distribution $p_{\mu^*}(\mathcal{P}|s)$ (with control policy $\mu^*$) over the paths of the MDP that is close to some given prior distribution $q(\mathcal{P}|s)$ \cite{peters2010relative,grau2018soft}. Here, the objective is to minimize the KL-divergence $D_{KL}(p_{\mu}\|q))$ under the constraint $J=J_0$ (as in (\ref{eq: OptimP1})).
\subsubsection{Non-dependence on choice of $J_0$ in (\ref{eq: OptimP1})}
In our framework we do not explicitly determine and work with the value of $J_0$. Instead we work with the Lagrange parameter $\beta$ in the Lagrangian $V^{\mu}_{\beta}(s)$ in (\ref{eq: Lag}) corresponding to the optimization problem (\ref{eq: OptimP1}). It is known from the sensitivity analysis \cite{jaynes1957information} that the small values of $\beta$ correspond to large values of $J_0$, and large values of $\beta$ correspond to small values of $J_0$. Thus, in our algorithms we solve the optimization problem (\ref{eq: OptimP1}) beginning at small values of $\beta=\beta_{\min}\approx 0$ (that corresponds to some feasible large $J_0$), and anneal it to a large value $\beta_{\max}$ (that corresponds to a small $J_0$ value) at which the stochastic policy $\mu$ in (\ref{eq: Policy}) converges to either $0$ or $1$. Also at $\beta\approx 0$, the stochastic policy $\mu_{\beta}^*$ in (\ref{eq: Policy}) follows a uniform distribution, which implicitly fixes the value of $J_0$. Therefore, the initial value of $J_0$ in the proposed algorithms are fixed and are not required to be pre-specified.
\subsubsection{Computational complexity} Our MEP-based Algorithm \ref{alg: Algorithm1} performs exactly the same number of computations as the Soft Q-learning algorithm \cite{fox2015taming} for each {\em epoch} (or, iteration) within an episode. In comparison to the Q and Double Q learning algorithms, our proposed algorithm, apart from performing the additional minor computations of explicitly determining $\mu^*$ in (\ref{eq: Policy}), exhibits the similar number of computational steps.
\subsubsection{Scheduling $\beta$ and Phase Transition} In our Algorithm \ref{alg: Algorithm1}, we follow a linear schedule $\beta_k=\sigma k$ ($\sigma>0$) as suggested in the benchmark algorithm \cite{fox2015taming} to anneal the parameter $\beta$. In the case of parameterized MDPs (Algorithm \ref{alg: Algorithm2}, and \ref{alg: Algorithm3}) we geometrically anneal $\beta$ (i.e. $\beta_{k+1}=\tau\beta_k$, $\tau>1$) from a small value $\beta_{\min}$ to a large value $\beta_{\max}$ at which the control policy $\mu_{\beta}^*$ converges to either $0$ or $1$. Several other MEP-based algorithms (that address problems akin to parameterized MDPs) such as Deterministic Annealing \cite{rose1991deterministic}, incorporate geometric annealing of $\beta$. The underlying idea in \cite{rose1991deterministic} is that the solution undergoes significant changes only at certain {\em critical} $\beta_{cr}$ ({\em phase transition}) and shows insignificant changes between two consecutive critical $\beta_{cr}$'s. Thus, for all practical purposes geometric annealing of $\beta$ works well. Similar to \cite{rose1991deterministic} our Algorithms \ref{alg: Algorithm2} and \ref{alg: Algorithm3} also undergo the phase transition and we are working on its analytical expression.
\subsubsection{Capacity and Exclusion Constraints} Certain parameterized MDPs may pose capacity or dynamical constraints over its parameters. For instance, each small cell $f_j$ allocated in the Figure \ref{fig: 5G_SC} can be constrained in capacity to cater to maximum $c_j$ fraction of user nodes in the network. Our framework allows to model such a constraint as $q_{\mu}(f_j):=\sum_{a,n_i}\mu(a|n_i)p(f_j|a,n_i)\leq c_j$ where $q_{\mu}(f_j)$ measures fraction of user nodes $\{n_i\}$ that connect to $f_j$. In another scenario, the locations $\{x_i\}$ of the user nodes could be dynamically varying as $\dot{x}_i=f(x,t)$. The resulting policy $\mu^*_{\beta}$ and small cells $\{y_j\}$ will also be time varying. We treat the free-energy function $V^{\mu}_{\beta}$ in (\ref{eq: free_energy}) as a {\em control-Lyapunov} function and determine time varying $\mu^*_{\beta}$ and $\{y_j\}$ such that $\dot{V}^{\mu}_{\beta}\leq 0$.
\subsubsection{Uncertainty in Parameters} Many application areas comprise of states and actions where the associated parameters are uncertain with a known distribution over the set of their possible values. For instance, a user nodes $n_i$ in Figure \ref{fig: 5G_SC} may have an associated uncertainty in its location $x_i$ owing to measurement errors. Our proposed framework easily incorporates such uncertainties in parameter values. For example, the above uncertainty will result into replacing $c(n_i,s',a)$ with $c'(n_i,s',a)=\sum_{x_i\in X_i}p(x_i|n_i)c(n_i,s',a)$ where $p(x_i|n_i)$ is the distribution over the set $X_i$ of location $x_i$. The subsequent solution approach remains the same as in Section \ref{sec: ParaProbSoln}.
\begin{appendices}
\section*{APPENDICES}
\noindent\appendixx{\label{app: AppDerivation}\bf Proof of Lemma \ref{lem: lem1}:} Let $\bar{x}_0=s$. By Assumption \ref{assum: assum1} $\exists$ a path $\omega=(\bar{u}_0,\bar{x}_1,\hdots,\bar{x}_N=\delta)$ such that $p_{\bar{\mu}}(\omega|x_0=s)>0$ which implies $p(x_{k+1}=\bar{x}_{k+1}|x_k=\bar{x}_k,u_k=\bar{u}_k)>0$ by (\ref{eq: Markov}). Then, probability $p_{\mu}(\omega|x_0=s)$ of taking path $\omega$ under the stochastic policy $\mu\in\Gamma$ in (\ref{eq: setGamma}) is also positive.\\
\noindent{\bf Proof of Theorem \ref{thm: HJBNonTrival}:} The following Lemma is needed
\begin{lemma}\label{lem: entropyRel}
The Shannon Entropy $H^{\mu}(\cdot)$ corresponding to the MDP illustrated in Section \ref{sec: MDP} satisfies the algebraic expression $\sum_{s'}p_{ss'}^aH^{\mu}(s') = \sum_{s'}p_{ss'}^a\log p_{ss'}^a+\log \mu_{a|s}+\lambda_s+1$. 
\end{lemma}
\begin{proof}
$H^{\mu}(\cdot)$ in (\ref{eq: OptimP1}) satisfies the recursive Bellman equation
\begin{align}
\text{\small$H^{\mu}(s') = \sum_{a's''}\mu_{a'|s'}p_{s's''}^{a'}\big[-\log p_{s's''}^{a'}- \log \mu_{a'|s'} + H^{\mu}(s'')\big]$}\nonumber 
\end{align}
On the right side of the above Bellman equation, we subtract a zero term $\lambda_{s'}(\sum_{a}\mu_{a|s'}-1)$ that accounts for normalization constraint $\sum_a\mu_{a|s'}=1$ and $\lambda_{s'}$ is some constant. Taking the derivative of the resulting expression we obtain
\begin{align}\label{eq: lab1}
\text{\small
$\frac{\partial H^{\mu}(s')}{\partial \mu_{a|s}} = \rho(s,a)\delta_{ss'}+ \sum_{a',s''}\mu_{a'|s'}p_{s's''}^{a'}\frac{\partial H^{\mu}(s'')}{\partial \mu_{a|s}}$}-\lambda_{s'}\delta_{ss'},
\end{align}
where $\rho(s,a) = -\sum_{s''}p_{ss''}^{a}(\log p_{ss''}^{a}-H^{\mu}(s''))-\log\mu_{a|s}-1$. The subsequent steps in the proof involve algebraic manipulations and makes use of the quantity $p_{\mu}(s'):=\sum_{s}p_{\mu}(s'|s)$ where $p_{\mu}(s'|s) = \sum_{a}p_{ss'}^a\mu_{a|s}$. Under the trivial assumption that for each state $s'$ there exists a state-action pair $(s,a)$ such that the probability of the system to enter the state $s'$ upon taking action $a$ in the state $s$ is non-zero ( i.e. $p_{ss'}^a>0$) we have that $p_{\mu}(s')>0$. Now, we multiply equation (\ref{eq: lab1}) by $p_{\mu}(s')$ and add over all $s'\in\mathcal{S}$ to obtain
\begin{align*}
&\text{\small
$\sum_{s'}p_{\mu}(s')\frac{\partial H^{\mu}(s')}{\partial \mu_{a|s}} = p_{\mu}(s')\rho(s,a)$}\\
&\qquad\qquad\text{\small$+ \sum_{s''}p_{\mu}(s'')\frac{\partial H^{\mu}(s'')}{\partial \mu_{a|s}}-p_{\mu}(s)\lambda_s$},
\end{align*}
where $p_{\mu}(s'') = \sum_{s'}p_{\mu}(s')p_{\mu}(s''|s')$. The derivative terms on both sides cancel to give $p_{\mu}(s')\rho(s,a)-\lambda_s=0$ which implies $\sum_{s'}p_{ss'}^{a}H^{\mu}(s') = \sum_{s'}p_{ss'}^a\log p_{ss'}^a+\log\mu_{a|s}+\lambda_s+1$. 
\end{proof}
Now consider the free energy function $V_{\beta}^{\mu}(s)$ in (\ref{eq: Lag}) and separate out the $t=0$ term in its infinite summation to obtain
\begin{align}\label{eq: recur1}
&\text{\small$V^{\mu}_{\beta}(s) = \mathbb{E}\Big[\sum_{t=0}^{\infty}\gamma^t c_{x_tx_{t+1}}^{u_t} + \frac{1}{\beta}\log p_{x_{t}x_{t+1}}^{u_{t}}$} \nonumber\\
&\qquad \text{\small$+ \frac{1}{\beta} \log \mu(u_t|x_t)\Big|x_0=s\Big]$}\nonumber\\
&\text{\small$V^{\mu}_{\beta}(s) = \sum_{s',a}\mu(a|s)p_{ss'}^a\big(c_{ss'}^a + \frac{1}{\beta}\log p_{ss'}^a + \frac{1}{\beta}\log \mu(a|s)\big)$}\nonumber\\
&\text{\small$+\mathbb{E}\Big[\sum_{t=1}^{\infty}\gamma^t c_{x_tx_{t+1}}^{u_t}+\frac{1}{\beta}\log p_{x_{t'}x_{t'+1}}^{u_{t'}} + \frac{1}{\beta} \log \mu(u_t|x_t)\Big|x_0=s\Big]$}\nonumber\\
&\text{let}\qquad t=t'+1,\qquad u'_{t'}=u_{t'+1},\qquad x'_{t'} = x_{t'+1}\nonumber\\
&\text{\small$\Rightarrow V_{\beta}^{\mu}(s)
= \sum_{s',a}\mu(a|s)p_{ss'}^a\big(c_{ss'}^a + \frac{1}{\beta}\log p_{ss'}^a + \frac{1}{\beta}\log \mu(a|s)\big)$}\nonumber\\
&\text{\small$+ \mathbb{E}\Big[\sum_{t'=0}^{\infty}\gamma^{t'+1}c_{x'_{t'}x'_{t'+1}}^{u'_{t'}}
+ \frac{1}{\beta}\log p_{x'_{t'}x'_{t'+1}}^{u'_{t'}}$}\nonumber\\
&\qquad+\text{\small$\frac{1}{\beta}\log \mu(u'_{t'}|x'_{t'})\Big|x_0=s\Big]$}\nonumber\\
&\text{\small$= \sum_{s',a}\mu(a|s)p_{ss'}^a\big(c_{ss'}^a + \frac{1}{\beta}\log p_{ss'}^a + \frac{1}{\beta}\log \mu(a|s)\big)$}\nonumber\\
&\qquad\text{\small$+ \gamma\mathbb{E}\Big[\sum_{t'=0}^{\infty}\gamma^{t'}c_{x'_{t'}x'_{t'+1}}^{u'_{t'}}+ \frac{1}{\gamma\beta}\log p_{x'_{t'}x'_{t'+1}}^{u'_{t'}}$}\nonumber\\
&\qquad\text{\small$+\frac{1}{\gamma\beta}\log \mu(u_{t'}'|x_{t'}')\Big|x_0=s\Big]$}\nonumber\\
&\text{\small$= \sum_{s',a}\mu(a|s)p_{ss'}^a\big(c_{ss'}^a +\frac{1}{\beta}\log p_{ss'}^a + \frac{1}{\beta}\log \mu(a|s)\big)$}\nonumber\\
&\qquad \text{\small$+ \gamma\mathbb{E}\Big[\mathbb{E}\Big[\sum_{t'=0}^{\infty}\gamma^{t'}c_{x'_{t'}x'_{t'+1}}^{u'_{t'}}+\frac{1}{\gamma\beta}\log p_{x'_{t'}x'_{t'+1}}^{a'_{t'}}$}\nonumber\\
&\text{\small$\qquad + \frac{1}{\gamma\beta}\log \mu(a_{t'}'|x_{t'}')\Big|x_0'=s'\Big]\Big|x_0=s\Big]$}\nonumber\\
&\text{\small$= \sum_{s',a}\mu(a|s)p_{ss'}^a\big(c_{ss'}^a +\frac{1}{\beta}\log p_{ss'}^a + \frac{1}{\beta}\log \mu(a|s)\big)$}\nonumber\\
&\qquad \text{\small$+ \gamma\sum_{a,s'}\mu(a|s)p_{ss'}^aV^{\mu}_{\gamma\beta}(s')$}\nonumber\\
&\text{\small$\Rightarrow V_{\beta}^{\mu}(s)= \sum_{a,s'}\mu_{a|s}p_{ss'}^{a}\Big[c_{ss'}^a+\frac{1}{\beta}\log p_{ss'}^a$}\nonumber\\
&\qquad \qquad \text{\small$+ \frac{1}{\beta}\log\mu_{a|s}+\gamma V_{\gamma\beta}^{\mu}(s')\Big],$}
\end{align}
where {\small$V_{\gamma\beta}^{\mu}(s') := J^\mu(s') -\frac{1}{\gamma\beta}H^{\mu}(s')$. Now we relate $V_{\gamma\beta}^{\mu}(s')$ with $V_{\beta}^{\mu}(s')$ by adding and subtracting $-\frac{1}{\beta}H^\mu(s')$ to $V_{\gamma\beta}^\mu (s')$. We obtain $V_{\gamma\beta}^{\mu}(s')= V_{\beta}^{\mu}(s') - \frac{1-\gamma}{\gamma\beta}H(s')$}. Substituting $V_{\gamma\beta}(s')$ and the algebraic expression obtained in Lemma \ref{lem: entropyRel} in the above equation (\ref{eq: recur1}) we obtain
{\small
\begin{align}
V^{\mu}_{\beta}(s) &= \sum_{s',a}\mu(a|s)p_{ss'}^a\big(c_{ss'}^a + \frac{1}{\beta}\log p_{ss'}^a +  \frac{1}{\beta}\log \mu(a|s)\big)\nonumber \\
&+ \gamma\sum_{a,s'}\mu(a|s)p_{ss'}^aV_{\beta}^{\mu}(s')  -\frac{1-\gamma}{\beta}\sum_{a,s'}\mu(a|s)p_{ss'}^a H^{\mu}(s')\nonumber
\end{align}
\begin{align}
&\Rightarrow V_{\beta}^{\mu}(s) = \sum_{s',a}\mu(a|s)p_{ss'}^a\big(c_{ss'}^a + \frac{1}{\beta}\log p_{ss'}^a +  \frac{1}{\beta}\log \mu(a|s)\big)\nonumber \\
&+ \gamma\sum_{a,s'}\mu(a|s)p_{ss'}^aV_{\beta}^{\mu}(s')-\frac{1-\gamma}{\beta}\sum_a \mu(a|s)\sum_{s'}p_{ss'}^aH^{\mu}(s') \nonumber
\end{align}
\begin{align}
&\Rightarrow V_{\beta}^{\mu}(s) = \sum_{s',a}\mu(a|s)p_{ss'}^a\big(c_{ss'}^a +\frac{1}{\beta}\log p_{ss'}^a +  \frac{1}{\beta}\log \mu(a|s)\big)\nonumber\\
&+ \gamma\sum_{a,s'}\mu(a|s)p_{ss'}^aV_{\beta}^{\mu}(s')\nonumber\\
&-\frac{1-\gamma}{\beta}\sum_a \mu(a|s)\Big(\sum_{s'}p_{ss'}^a\log p_{ss'}^a + \log\mu(a|s) + \lambda_s+1\Big)\nonumber
\end{align}
\begin{align}
\Rightarrow V_{\beta}^{\mu}(s) &= \sum_{a,s'}\mu(a|s)p_{ss'}^a\big(c_{ss'}^a + \frac{\gamma}{\beta}\log p_{ss'}^a + \frac{\gamma}{\beta}\log \mu(a|s)\big)\nonumber\\
&+ \gamma\sum_{a,s'}\mu(a|s)p_{ss'}^a V_{\beta}^{\mu}(s')+c_0(s)
\end{align}}
where $c_0(s)=-\frac{1-\gamma}{\beta}(\lambda_s+1)$ does not depend on the policy $\mu$. Therefore, since the control policy $\mu^*(a|s)$ is determined by taking critical points of $V_\beta^\mu(s)$, it is independent of $c_0(s)$ as shown in the following subsection.

\subsection{Independence of policy on $c_0$}

The Bellman equation is given by 
\begin{align}\label{eq: recurBell}
V_{\beta}^{\mu}(s) &= \sum_{a,s'}\mu(a|s)p_{ss'}^a\big(c_{ss'}^a + \frac{\gamma}{\beta}\log p_{ss'}^a + \frac{\gamma}{\beta}\log \mu(a|s)\big)\nonumber\\
&+ \gamma\sum_{a,s'}\mu(a|s)p_{ss'}^a V_{\beta}^{\mu}(s')+c_0(s)
\end{align}
The optimal control policy $\mu_\beta^*(a|s)$ obtained upon taking the derivative of the recursive Bellman equation (and also accounting for $\sum_{a}\mu(a|s)=1$) is given by 
\begin{align}
\mu_{\beta}^*(a|s) &= \frac{\exp\{-\frac{\beta}{\gamma} \bar{\Lambda}^*_{\beta}(s,a)\}}{\sum_{a'}\exp\{-\frac{\beta}{\gamma} \bar{\Lambda}^*_{\beta}(s,a')\}},\quad \text{where}\label{eq: OptimalCtrl1}\\
\bar{\Lambda}_{\beta}^*(s,a)&=\sum_{s'}p_{ss'}^a\big(c_{ss'}^a+\frac{\gamma}{\beta}\log p_{ss'}^a + \gamma V_{\beta}^{*}(s')\big)\label{eq: OptimalCtrl2},
\end{align}

Substituting the optimal policy $\mu_{\beta}^*(a|s)$ (\ref{eq: OptimalCtrl1}) into the recursive Bellman in (\ref{eq: recurBell}) we obtain 
\begin{align}\label{eq: optrecurLag}
V_{\beta}^*(s) = -\frac{\gamma}{\beta}\log \sum_{a}\exp\Big\{-\frac{\beta}{\gamma}\bar{\Lambda}_{\beta}^*(s,a)\Big\} + c_0(s).
\end{align}
Substituting the above equation (\ref{eq: optrecurLag}) into the state-action value function in (\ref{eq: OptimalCtrl2}) we obtain the following map
\begin{align}\label{eq: mapLambda}
&\bar{\Lambda}_{\beta}^*(s,a) = \sum_{s'}p_{ss'}^a\Big[\big(c_{ss'}^a + \frac{\gamma}{\beta}\log p_{ss'}^a\big)\nonumber\\
&- \frac{\gamma^2}{\beta}\log\sum_{a}\exp\Big\{-\frac{\beta}{\gamma}\bar{\Lambda}_{\beta}^*(s,a)\Big\}\Big]+ c_0(s)=: [T_1\bar{\Lambda}_{\beta}](s,a),
\end{align}
where the proof that the map $T_1:\bar{\Lambda}_{\beta}\rightarrow \bar{\Lambda}_{\beta}$ is a contraction map is analogous to the proof of Theorem \ref{thm: ContraMapFinite}.\\

The state-action value $\Lambda_{\beta}^*(s,a)$ obtained by disregarding $c_0(s)$ in (\ref{eq: optrecurLag}) satisfies the recursive equation
\begin{align}\label{eq: mapLambda1}
&\Lambda_{\beta}^*(s,a) = \sum_{s'}p_{ss'}^a\Big[\big(c_{ss'}^a + \frac{\gamma}{\beta}\log p_{ss'}^a\big)\nonumber\\
&- \frac{\gamma^2}{\beta}\log\sum_{a}\exp\Big\{-\frac{\beta}{\gamma}\Lambda_{\beta}^*(s,a)\Big\}\Big]=:[T\Lambda_{\beta}](s,a),
\end{align}
where the map $T:\Lambda_{\beta}\rightarrow \Lambda_{\beta}$ has been shown to be a contraction map in the Theorem \ref{thm: ContraMapFinite}.\\

{\em Remark: }Note that $[T_1x](s,a)=[Tx](s,a)+c_0(s)$. Therefore, $T_1$ is a contraction since $T$ is contraction and has a unique fixed point.

\textbf{Claim:} The optimal control policy $\mu_{\beta}^*(a|s)$ in (\ref{eq: OptimalCtrl1}) is same if computed using either $\bar{\Lambda}_{\beta}^*(s,a)$ or $\Lambda_{\beta}^*(s,a)$. This is so because the fixed points $\bar{\Lambda}_{\beta}^*(s,a)$ and $\Lambda_{\beta}^*(s,a)$ of the contraction maps $T_1:\bar{\Lambda}_{\beta}\rightarrow \bar{\Lambda}_{\beta}$ and $T:\Lambda_{\beta}\rightarrow\Lambda_{\beta}$, respectively, differ by $\frac{1}{1-\gamma}c_0(s)$, i.e. $\bar{\Lambda}_{\beta}^*(s,a) = \Lambda_{\beta}^*(s,a) + \frac{1}{1-\gamma}c_0(s)$. If we plug the above relation into the control policy in (\ref{eq: OptimalCtrl1}) we obtain
\begin{align}
\mu_{\beta}^*(a|s) &= \frac{\exp\{-\frac{\beta}{\gamma} {\Lambda}^*_{\beta}(s,a)-\frac{\beta/\gamma}{1-\gamma}c_0(s)\}}{\sum_{a'}\exp\{-\frac{\beta}{\gamma} {\Lambda}^*_{\beta}(s,a')-\frac{\beta/\gamma}{1-\gamma}c_0(s)\}}\nonumber\\
\Rightarrow \mu_{\beta}^*(a|s) &= \frac{\exp\{-\frac{\beta}{\gamma} {\Lambda}^*_{\beta}(s,a)\}}{\sum_{a'}\exp\{-\frac{\beta}{\gamma} {\Lambda}^*_{\beta}(s,a')\}}\nonumber
\end{align}
Thus, indicating that we do not need to take $c_0(s)$ into account while computing the optimal policy.

{\bf Proof for $\bar{\Lambda}_{\beta}^*(s,a) = \Lambda_{\beta}^*(s,a) + \frac{1}{1-\gamma}c_0(s)$: } This can be checked by substituting this expression into the definition of the map $[T_1\bar{\Lambda}_{\beta}]$ in (\ref{eq: mapLambda}); we get
\begin{align}
&T_1\bar{\Lambda}_{\beta}^*(s,a) = \sum_{s'}p_{ss'}^a\Big[\big(c_{ss'}^a + \frac{\gamma}{\beta}\log p_{ss'}^a\big)\nonumber\\
&- \frac{\gamma^2}{\beta}\log\sum_{a}\exp\Big\{-\frac{\beta}{\gamma}\Lambda_{\beta}^*(s,a)-\frac{\beta/\gamma}{1-\gamma}c_0(s)\Big\}\Big] + c_0(s)\nonumber
\end{align}
\begin{align}
&T_1\bar{\Lambda}_{\beta}^*(s,a) = \sum_{s'}p_{ss'}^a\Big[\big(c_{ss'}^a + \frac{\gamma}{\beta}\log p_{ss'}^a\big)\nonumber\\
&- \frac{\gamma^2}{\beta}\log\sum_{a}\exp\Big\{-\frac{\beta}{\gamma}\Lambda_{\beta}^*(s,a)\Big\}\Big]+\frac{\gamma}{1-\gamma}c_0(s) + c_0(s)\nonumber\\
&T_1\bar{\Lambda}_{\beta}^*(s,a) =\sum_{s'}p_{ss'}^a\Big[\big(c_{ss'}^a + \frac{\gamma}{\beta}\log p_{ss'}^a\big)\nonumber\\- &\frac{\gamma^2}{\beta}\log\sum_{a}\exp\Big\{-\frac{\beta}{\gamma}\Lambda_{\beta}^*(s,a)\Big\}\Big]+\frac{1}{1-\gamma}c_0(s)\nonumber\\
&T_1\bar{\Lambda}_{\beta}^*(s,a) = T{\Lambda}_{\beta}^*(s,a) + \frac{1}{1-\gamma}c_0(s)\nonumber\\
&\bar{\Lambda}_{\beta}^*(s,a) = \Lambda_{\beta}^*(s,a) + \frac{1}{1-\gamma}c_0(s)\nonumber
\end{align}
Hence proved.

\noindent\appendixx{\label{app: AppCont_MapQ}\bf Proof of Theorem \ref{thm: ContraMapFinite}: }Following lemma is used.
\begin{lemma}\label{lem: alpha_xi}
For every policy $\mu\in\Gamma$ defined in (\ref{eq: setGamma}) there exists a vector $\xi=(\xi_s)\in\mathbb{R}_{+}^{|\mathcal{S}|}$ with positive components and a scalar $\lambda < 1$ such that $\sum_{s'}p_{ss'}^a\xi_{s'}\leq \lambda\xi_s$ for all $s\in\mathcal{S}$ and $a\in\mathcal{A}$.
\end{lemma}
\begin{proof}
Consider a new MDP with state transition probabilities similar to the original MDP and the transition costs $c_{ss'}^a = -1-\frac{1}{\beta}\log(|\mathcal{A}||\mathcal{S}|)$ except when $s=\delta$. 
Thus, the free-energy function $V_{\beta}^{\mu}(s)$ in (\ref{eq: Lag}) for the new MDP is less than or equal to $-1$. We define $-\xi_s\triangleq V_{\beta}^*(s)$ (as given in \ref{eq: free_energy})) and use LogSumExp \cite{kobayashi2007mathematics} inequality to obtain $-\xi_s\leq\min_{a}\Lambda_{\beta}(s,a)\leq \Lambda_{\beta}(s,a) ~\forall~a\in\mathcal{A}$ where $\Lambda_{\beta}(s,a)$ is the state action value function in (\ref{eq: Q_map}). Thus, $-\xi_s\leq \sum_{s'}p_{ss'}^{a}\big(c_{ss'}^a + \frac{\gamma}{\beta}\log p_{ss'}^a-\gamma\xi_{s'}\big)$ and upon substituting $c_{ss'}^a$ we obtain $-\xi_s\leq -1-\gamma\sum_{s'}p_{ss'}^a\xi_{s'}\leq -1 - \sum_{s'}p_{ss'}^a\xi_{s'}$.
\begin{align*}
\text{\small
$\Rightarrow \sum_{s'\in\mathcal{S}}p_{ss'}^a\xi_{s'}\leq \xi_{s}-1 \leq \Big[\max_{s}\frac{\xi_s-1}{\xi_s}\Big]\xi_s=:\lambda\xi_s$}.
\end{align*}
Since $V_{\beta}^*(s)\leq-1$ $\Rightarrow$ $\xi_s-1\geq0$ and thus $\lambda < 1$. 
\end{proof}
Next we show that $T:\Lambda_{\beta}\rightarrow \Lambda_{\beta}$ in (\ref{eq: Q_map}) is a contraction map. For any $\hat{\Lambda}_{\beta}$ and $\check{\Lambda}_{\beta}$ we have that $[T\hat\Lambda_{\beta}-T\check\Lambda_{\beta}](s,a)$
\begin{align}\label{eq: LogSumEq}
&\text{\small$=-\frac{\gamma^2}{\beta}\sum_{s'\in\mathcal{S}}p_{ss'}^{a}\log\frac{\sum_{a}
\exp{\big(-\frac{\beta}{\gamma}\hat{\Lambda}_{\beta}(s',a)\big)}}{\sum_{a'}\exp{\big(-\frac{\beta}{\gamma}\check{\Lambda}_{\beta}(s',a')\big)}}$}\nonumber\\
&\text{\small$\geq\gamma \sum_{s',a'}p_{ss'}^a\hat\mu_{a'|s'}(\hat\Lambda_{\beta}(s',a')-\check\Lambda_{\beta}(s',a')) =:\gamma\Delta_{\hat\mu}$},
\end{align}
where we use the Log sum inequality to obtain (\ref{eq: LogSumEq}), and $\hat\mu_{a|s}$ is the stochastic policy in (\ref{eq: Policy}) corresponding to $\hat\Lambda_{\beta}(s,a)$.  Similarly, we obtain $[T\check{\Lambda}_{\beta} - T\hat{\Lambda}_{\beta}](s,a)\geq-\gamma\sum_{s',a'}p_{ss'}^{a}\check{\mu}_{a'|s'}(\hat{\Lambda}_{\beta}(s',a')-\check{\Lambda}_{\beta}(s',a'))=:-\gamma\Delta_{\check\mu}$ where $\check\mu_{a|s}$ is the policy in (\ref{eq: Policy}) corresponding to $\check\Lambda_{\beta}(s,a)$. Now from $\gamma\Delta_{\hat\mu}\leq[T\hat\Lambda_{\beta}-T\check\Lambda_{\beta}](s,a)\leq \gamma\Delta_{\check\mu}$ we conclude that $|[T\hat{\Lambda}_{\beta}-T\check{\Lambda}_{\beta}](s,a)|\leq \gamma\Delta_{\bar\mu}(s,a)$ where $\Delta_{\bar\mu}(s,a)=\max\{|\Delta_{\hat\mu}(s,a)|,|\Delta_{\check\mu}(s,a)|\}$ and we have $|[T\hat{\Lambda}_{\beta}-T\check{\Lambda}_{\beta}](s,a)|$
\begin{align}
&\text{\small$\leq \gamma\sum_{s',a'}p_{ss'}^a\bar{\mu}_{a'|s'}|\hat\Lambda_{\beta}(s',a')-\check\Lambda_{\beta}(s',a')|$}\label{eq: AppC1}\\
&\text{\small$\leq \gamma\sum_{s',a'}p_{ss'}^a\xi_{s'}\bar{\mu}_{a'|s'}\|\hat\Lambda_{\beta}-\check\Lambda_{\beta}\|_{\xi}$}\label{eq: AppC3}
\end{align}
where $\|\Lambda_{\beta}\|_{\xi}=\max_{s,a}\frac{\Lambda_{\beta}(s,a)}{\xi_s}$ and $\xi\in\mathbb{R}^{\mathcal{S}}$ is as given in Lemma \ref{lem: alpha_xi}. Further, from the same Lemma we obtain
\begin{align}
\text{\small $|[T\hat{\Lambda}_{\beta}-T\check{\Lambda}_{\beta}](s,a)|\leq \gamma\lambda\xi_{s}\sum_{a'\in\mathcal{A}}\bar{\mu}_{a'|s'}\|\hat{\Lambda}_{\beta}-\check{\Lambda}_{\beta}\|_{\xi}$}\label{eq: AppC4}\\
\text{\small $\Rightarrow \|T\hat{\Lambda}_{\beta}-T\check{\Lambda}_{\beta}\|_{\xi}\leq \gamma\lambda\|\hat{\Lambda}_{\beta}-\check{\Lambda}_{\beta}\|_{\xi} \text{ with }\gamma\lambda<1$}\label{eq: AppC5}
\end{align}

\noindent\appendixx{\label{app: InfiniteSE}\bf Proof of Theorem \ref{thm: BellmanTrue2}:} The proof follows the similar idea as the proof for Theorem \ref{thm: HJBNonTrival} in Appendix \ref{app: AppDerivation} and thus, we do not explain it in detail except the following Lemma that illustrates the algebraic structure of the discounted Shannon entropy $H_d^{\mu}(\cdot)$ in (\ref{eq: DiscountShannonEnt}) which is different from that in Lemma \ref{lem: entropyRel} and also required in our proof of the said theorem.
\begin{lemma}\label{lem: entropyRel2}
The discounted Shannon entropy $H_d^{\mu}(\cdot)$ corresponding to the MDP in Section \ref{sec: InfiMDP} satisfies the algebraic term $\alpha\sum_{s'}p_{ss'}^a H^{\mu}_d(s') = \sum_{s'}p_{ss'}^a\log p_{ss'}^a+\log \alpha\mu(a|s) + \lambda_s+1$.
\end{lemma}
\begin{proof}
Define a new MDP that augments the action and state spaces ($\mathcal{A},\mathcal{S}$) of the original MDP with an additional action $a_e$ and state $s_e$, respectively, and derives its state-transition probability $\{q_{ss'}^a\}$ and policy $\{\zeta_{a|s}\}$ from original MDP as
\begin{footnotesize}
\begin{align*}\label{eq: augmentedMDP}
q_{ss'}^a=
\begin{cases}
p_{ss'}^a & \forall s,s'\in\mathcal{S},a\in\mathcal{A}\\
1 & \text{if } s',a=s_e,a_e\\
1 & \text{if } s'=s=s_e\\
0 & \text{otherwise}
\end{cases}
\zeta_{a|s}=
\begin{cases}
\alpha\mu_{a|s} & \forall (s,a)\in(\mathcal{S},\mathcal{A})\\
1-\alpha & \text{if } a = a_e, s\in \mathcal{S}\\
0 & \text{if } a\in\mathcal{A}, s=s_e\\
1 & \text{if } a = a_e, s = s_e
\end{cases}
\end{align*}
\end{footnotesize}
\noindent Next, we define $T^{\mu}:=\alpha H^{\mu}_d$ that satisfies $\text{\small
$T^{\mu}(s') = \sum_{a's''}\eta_{a'|s'}p_{s's''}^{a'}\big[-\log p_{s's''}^a-\log \eta_{a'|s'} + T^{\mu}(s'')\big]$}$ derived using (\ref{eq: DiscountShannonEnt}) where $\eta_{a'|s'}=\alpha\mu_{a'|s'}$. The subsequent steps of the proof are same as the proof of Lemma \ref{lem: entropyRel}. 
\end{proof}

\noindent\appendixx{\label{app: Q_learnConv}\em Proof of Proposition \ref{pro: Proposition1}: }The proof in this section is analogous to the proof of Proposition 5.5 in \cite{bertsekas1996neuro}. Let $\bar{T}$ be the map in (\ref{eq: Q_map2}). The stochastic iterative updates in (\ref{eq: Q_upd1}) can be re-written as $\scriptstyle\bar{\Psi}_{t+1}(x_t,u_t)=(1-\nu_t(x_t,u_t))\bar{\Psi}_t(x_t,u_t)+\nu_t(x_t,u_t)\big([\bar{T}\bar{\Psi}_t](x_t,u_t)
+ w_t(x_t,u_t)\big)$ where $\scriptstyle w_t(x_t,u_t) = c_{x_tx_{t+1}}^{u_t}-
\frac{\gamma^2}{\beta}\log\sum_{a}\exp(-\frac{\beta}{\gamma} \bar{\Psi}_t(s_{t+1},a)) - \bar{T}\bar{\Psi}_t(x_t,u_t)$. Let $\mathcal{F}_t$ represent the history of the stochastic updates, i.e.,
$\scriptstyle\mathcal{F}_t=\{\bar{\Psi}_0,\hdots,\bar{\Psi}_t,w_0,\hdots,w_{t-1},\nu_0,\hdots,\nu_t\},$ then $\scriptstyle\mathbb{E}[w_t(x_t,u_t)|\mathcal{F}_t]=0$ and $\scriptstyle\mathbb{E}[w_t^2(x_t,u_t)|\mathcal{F}_t]\leq K(1+\max_{s,a}\bar{\Psi}_t^2(s,a))$,
where $K$ is a constant. These expressions satisfy the conditions on the expected value and the variance of $w_t(x_t,u_t)$ that along with the contraction property of $\bar{T}$ guarantees the convergence of the stochastic updates (\ref{eq: Q_upd1}) as illustrated in the Proposition 4.4 in \cite{bertsekas1996neuro}.

\noindent{\bf Proof of Theorem \ref{thm: ParaDerivatives}: }We show that the map $T_1$ in (\ref{eq: V_zeta_j}) is a contraction map. For any $K_{\zeta_s}^{\beta}$ and $\bar{K}_{\zeta_s}^{\beta}$ we obtain that $|[T_1K_{\zeta_{s}}^{\beta}-T_1\bar{K}_{\zeta_{s}}^{\beta}](s')|\leq \gamma\sum_{a,s''}p_{s's''}^{a}\mu_{a|s'}|K_{\zeta_{s}}^{\beta}(s'',a)-\bar{K}_{\zeta_{s}}^{\beta}(s'',a)|$. Note that this inequality is similar to the one in (\ref{eq: AppC1}); thus, we follow the exact same steps from (\ref{eq: AppC1}) to (\ref{eq: AppC5}) to show that $\|T_1K_{\zeta_s}^{\beta}-T_1\bar{K}_{\zeta_s}^{\beta}\|_{\xi}\leq \gamma\lambda\|K_{\zeta_s}^{\beta}-\bar{K}_{\zeta_s}^{\beta}\|_{\xi}$ and $\gamma\lambda<1$.

\noindent{\bf Proof of Proposition \ref{pro: Proposition2}: }The proof in this section is similar to the proof of Proposition \ref{pro: Proposition1} in Appendix \ref{app: Q_learnConv}. Additional conditions on the boundedness of the derivatives $\big|\frac{\partial c_{ss'}^a}{\partial \zeta_l}\big|$ and $\big|\frac{\partial c_{ss'}^a}{\partial \eta_k}\big|$ are required to bound the variance $\mathbb{E}[w_t^2|\mathcal{F}_t]$.
\end{appendices}

\ifCLASSOPTIONcaptionsoff
  \newpage
\fi



\bibliographystyle{IEEEtran}
\bibliography{IEEEabrv}

\begin{IEEEbiography}
[{\includegraphics[width=1in,height=1.25in,clip,keepaspectratio]{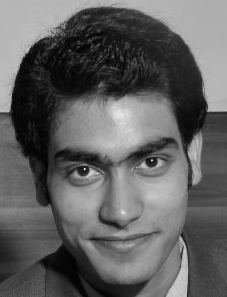}}]{Amber Srivastava}
received the B.Tech degree in Mechanical Engineering from the Indian Institute of Technology, Kanpur, in 2014 after which he held employment with an FMCG in India for an year (2014-2015) as an Assistant Manager. He obtained his Masters in Mathermatics from University of Illinois at Urbana Champaign (UIUC) in 2020. Currently he is a PhD student in the Mechanical Science and Engineering department of UIUC. His areas of interest are optimization, learning and controls.\vspace{-0.5cm}
\end{IEEEbiography}
\begin{IEEEbiography}
[{\includegraphics[width=1in,height=1.25in,clip,keepaspectratio]{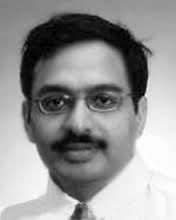}}]{Srinivasa M. Salapaka}
received the B. Tech from the IIT Madras, in 1995 and MS and PhD degrees in Mechanical from the University of California, Santa Barbara, in 1997 and 2002, respectively. From 2002 to 2004, he was a postdoctoral associate at Massachusetts Institute of Technology. Since 2004, he has been a faculty in the Mechanical Science and Engineering Department, University of Illinois, Urbana-Champaign. His areas of current research are controls for nanotechnology, combinatorial optimization, and power electronics.
\end{IEEEbiography}

\pagebreak
\onecolumn

\section*{\Large \textbf{Supplementary Material for CYB-E-2020-05-1237}}
\vspace{0.5cm}

\section{Link to Source Code}\label{sec: linkToSource}
Please find the sample codes for the model-free RL, and parameterized MDP and RL setting in either of the links below:
\begin{enumerate}
\item \url{https://drive.google.com/drive/folders/1fVA3nM1Oh6drQpi3H50ZPTDppK8Tv5bX?usp=sharing}
\item \url{https://uofi.box.com/s/twbo129bmst0e24ykgr1rhwwc8kyhe1o}
\end{enumerate}

\section{Simulations on the Double Chain Environment}\label{sec: DoubleChainSM}

In this section we demonstrate our Algorithm \ref{alg: Algorithm1} to determine the control policy $\mu^*$ for the finite and infinite Shannon entropy variants of the double chain (DC) environments in Figure \ref{fig: DoubleChainEnvs}. 

\begin{figure}[h]
    \centering
    \includegraphics[width=0.45\columnwidth]{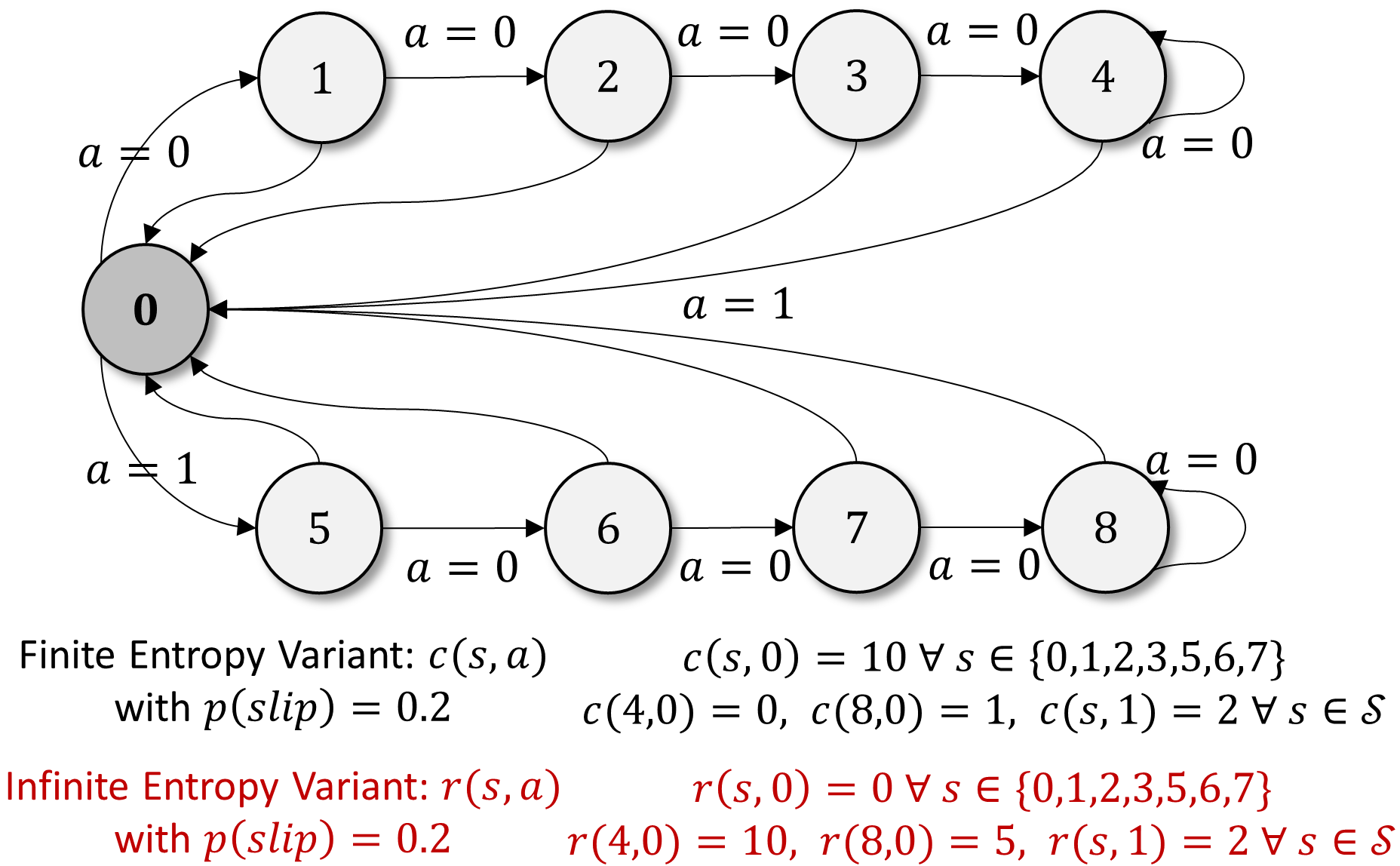}
    \caption{Finite and infinite entropy variants (respectively, cost and reward variants) of Double Chain (DC) environment. State space $\mathcal{S}=\{0,1,\hdots,8\}$ and action space $\mathcal{A}=\{0,1\}$. Cost $c(s,a)$ for version (a) is denoted in black and reward $r(s,a)$ for version (b) is in red. The agent slips to the state $s'=0$ with probability $0.2$ every time the action $a=0$ is taken at the state $s\in\mathcal{S}$\textbackslash$\{4\}$ in the finite entropy variant (and $s\in\mathcal{S}$ for the infinite entropy version (b)). The state $s=4$ is the cost-free termination state for version (a).}
    \label{fig: DoubleChainEnvs}
\end{figure}

{\em Faster Convergence to Optimal $J^*$: } Figures \ref{fig: CD_Simula}(a1)-(a4) (finite entropy variant of DC) and Figures \ref{fig: CD_Simula} (d1)-(d4) (infinite entropy variant of DC) illustrate the faster convergence of our MEP-based Algorithm \ref{alg: Algorithm1} for different discount factor $\gamma$ values. Here the learning error is given by $\Delta V$ at each episode between the value function $V_{\beta}^{\mu}$ corresponding to policy $\mu=\mu(ep)$ in the episode $ep$ and the optimal value function $J^*$; that is, 
\begin{align}
\Delta V(ep) = \frac{1}{N}\sum_{i=1}^N\sum_{s\in\mathcal{S}}\big|V_{\beta,i}^{\mu(ep)}(s)-J^*(s)\big|,
\end{align}
where $N$ denotes the total experimental runs and $i$ indexes the value function $V_{\beta,i}^{\mu}$ for each run. As observed in Figures \ref{fig: CD_Simula}(a1) and \ref{fig: CD_Simula}(a3), and Figures \ref{fig: CD_Simula}(d1) and \ref{fig: CD_Simula}(d3), our Algorithm \ref{alg: Algorithm1} converges even faster as the discount factor $\gamma$ decreases. We characterize the faster convergence rates also in terms of the convergence time - more precisely the  {\em critical} episode $\bar{E}_{pr}$ beyond which learning error $\Delta V$ is within $5\%$ of the optimal $J^*$ (see  Figures \ref{fig: CD_Simula}(a5) and \ref{fig: CD_Simula}(d5)). As is observed in the figures, the performance of our (MEP-based) algorithms gets even better with decreasing $\gamma$ values as the stochastic policy $\mu_{\beta}^*$ optimizes the amount of exploration required along the paths based on $\gamma$ values, whereas the other algorithms show deterioration in performance as $\gamma$ decreases.

\begin{figure}
    \centering
    \includegraphics[width=0.78\columnwidth]{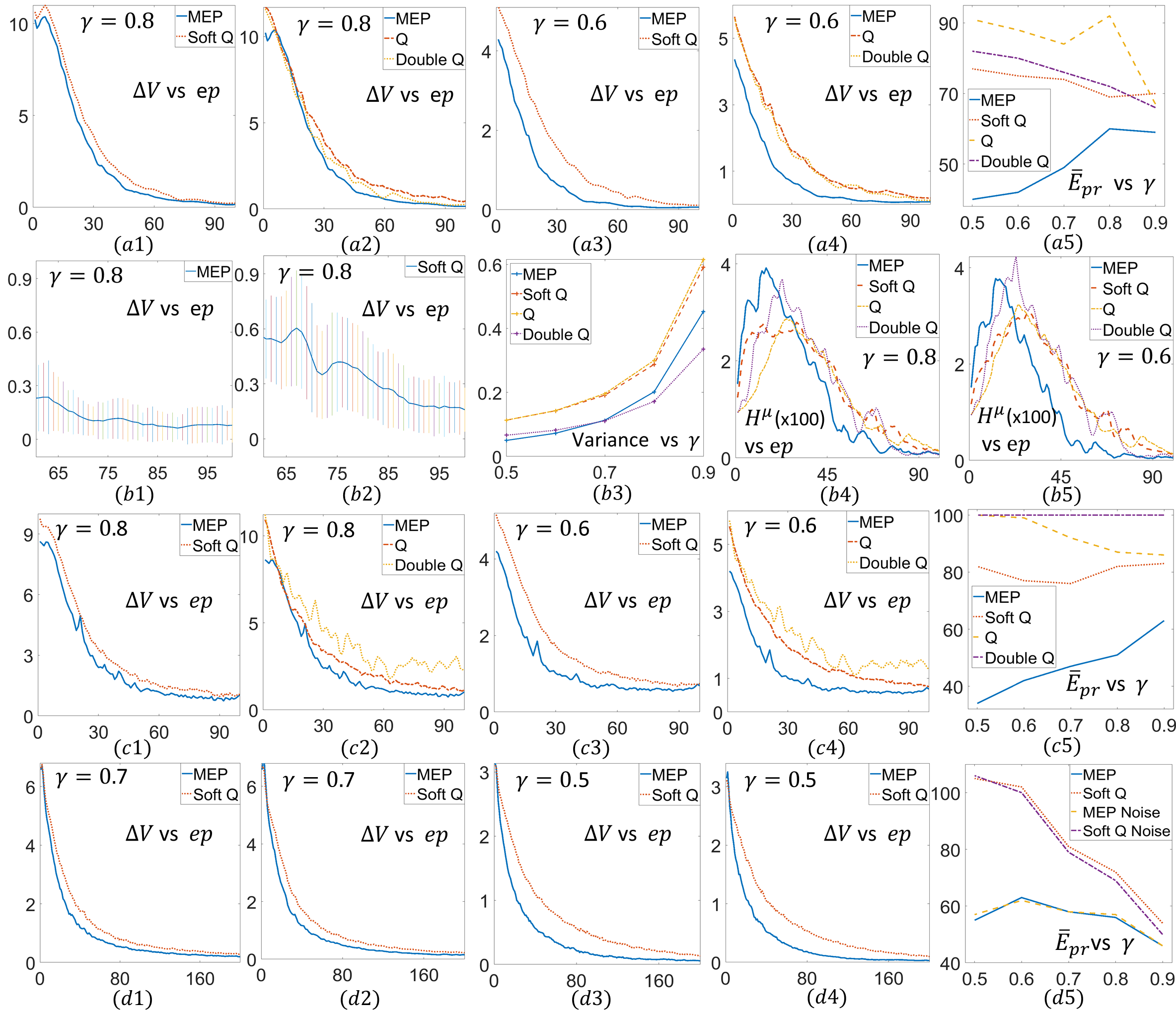}
    \caption{Performance of MEP-based algorithm: Illustrations on Double Chain Environment in Figure \ref{fig: DoubleChainEnvs}. (a1)-(a4) Finite Entropy variant:  Illustrates faster convergence of Algorithm \ref{alg: Algorithm1} (MEP) at different $\gamma$ values. (a5) Demonstrates faster and increasing rates of convergence of Algorithm \ref{alg: Algorithm1} (MEP) with decreasing $\gamma$ values. (b1)-(b3) Illustrates higher variance of Soft Q in comparison to MEP. (b4)-(b5) MEP exhibits larger Shannon Entropy (equivalently, exploration) during early episodes and smaller Shannon Entropy (equivalently, more exploitation) in later episodes; together these are responsible for faster convergence of MEP as illustrated in Fig. \ref{fig: CD_Simula}(a1)-(a5). (c1)-(c5) Cost Version (with added Gaussian noise) : Similar observations as in (a1)-(a5) with significantly higher instability in learning with Double Q algorithm. (d1)-(d5) {\em Infinite Entropy Variant} : Demonstrates faster convergence of Algorithm \ref{alg: Algorithm1} (MEP) to $J^*$ without noise in (d1),(d3) and with added noise in (d2),(d4). (d5) Illustrates the consistent faster convergence rates of MEP across different $\gamma$ values.}
    \label{fig: CD_Simula}
\end{figure}

{\em Smaller Variance: }Figures \ref{fig: CD_Simula}(b1)-(b3) compare the variance observed in the learning process during the last $40$ epsiodes of the $N$ experimental runs. As demonstrated, the Soft Q (Fig. \ref{fig: CD_Simula}(b2)) algorithm exhibit higher variances when compared to to the MEP-based (Fig. \ref{fig: CD_Simula}(b1)). The Figure \ref{fig: CD_Simula}(b3) plots the {\em average} variances of all the algorithms in the last $40$ episodes of the learning process with respect to different values of $\gamma$. As illustrated, Soft Q and Q learning exhibit higher variances across all $\gamma$ values in comparison to our and Double Q algorithms. Between our and Double Q, the variance in our algorithm is smaller at lower values of $\gamma$ and vice-versa.

{\em Inherently better exploration and exploitation: } The Shannon entropy $H^{\mu}$ corresponding to our algorithm is higher during the initial episodes indicating a better exploration under the learned policy $\mu$, when compared to other algorithms as seen in Figures \ref{fig: CD_Simula}(b4)-(b5). Additionally, it exhibits smaller $H^{\mu}$ in the later episodes indicating a more exploitative nature of the learned policy $\mu$. Unbiased policies resulting from our algorithm along with enhanced exploration-exploitation trade-off results into the faster convergence rates, and smaller variance as  discussed above.

{\em Robustness to noise in data: }Figures \ref{fig: CD_Simula}(c1)-(c5) demonstrate robustness to noisy environments; here the instantaneous cost $c(s,a)$ in the finite horizon variant of DC is noisy. For the purpose of simulations, we add Gaussian noise $\mathcal{N}(0,\sigma^2)$ with $\sigma=5$ when $a=1,s\in\mathcal{S}$, $a=0,s=8$, otherwise $\sigma=10$. Similar to our observations and conclusions in Figures \ref{fig: CD_Simula}(a1)-(a4), Figures \ref{fig: CD_Simula}(d1)-(d2) ($\gamma=0.8$) and Figures \ref{fig: CD_Simula}(d3)-(d4) ($\gamma=0.6$) illustrate faster convergence of our MEP-based algorithm. Also, similar to the Figure \ref{fig: CD_Simula}(a5), Figure \ref{fig: CD_Simula}(d5) exhibits higher convergence rates (i.e. lower $\bar{E}_{pr}$) $\forall$ $\gamma$ values in comparison to benchmark algorithms. In the infinite entropy variant of DC, results of similar nature are obtained upon adding the noise $\mathcal{N}(0,\sigma^2)$ as illustrated in Figure \ref{fig: CD_Simula}(d2) and \ref{fig: CD_Simula}(d4).

\section{Comparison with MIRL \cite{grau2018soft} and REPS \cite{peters2010relative}}\label{sec: MIRLvsREPS}

In this section we demonstrate the benefit of considering distributions over the {\em paths of an MDP} over considering distribution over the actions. Figure \ref{fig:CompWithMIRL}(a1)-(a4) compare the mutual information regularization (MIRL) algorithm presented in \cite{grau2018soft} (where mutual information is regularized with same discount factor as instantaneous cost) with the mutual information of distribution over the paths of MDP (our MEP-based idea that results into {\em no} discount factor for the regularizer term). Note the faster convergence plots $\frac{\Delta V}{J^*}$ versus $ep$ plots in Figure \ref{fig:CompWithMIRL}(a1)-a(3). In the Figure \ref{fig:CompWithMIRL}(a4) we demonstrate the faster convergence rate ($\bar{E}_{pr}$ versus $\gamma$) obtained when considering distribution over the paths of the MDP.
\begin{figure}
    \centering
    \includegraphics[width=0.7\columnwidth]{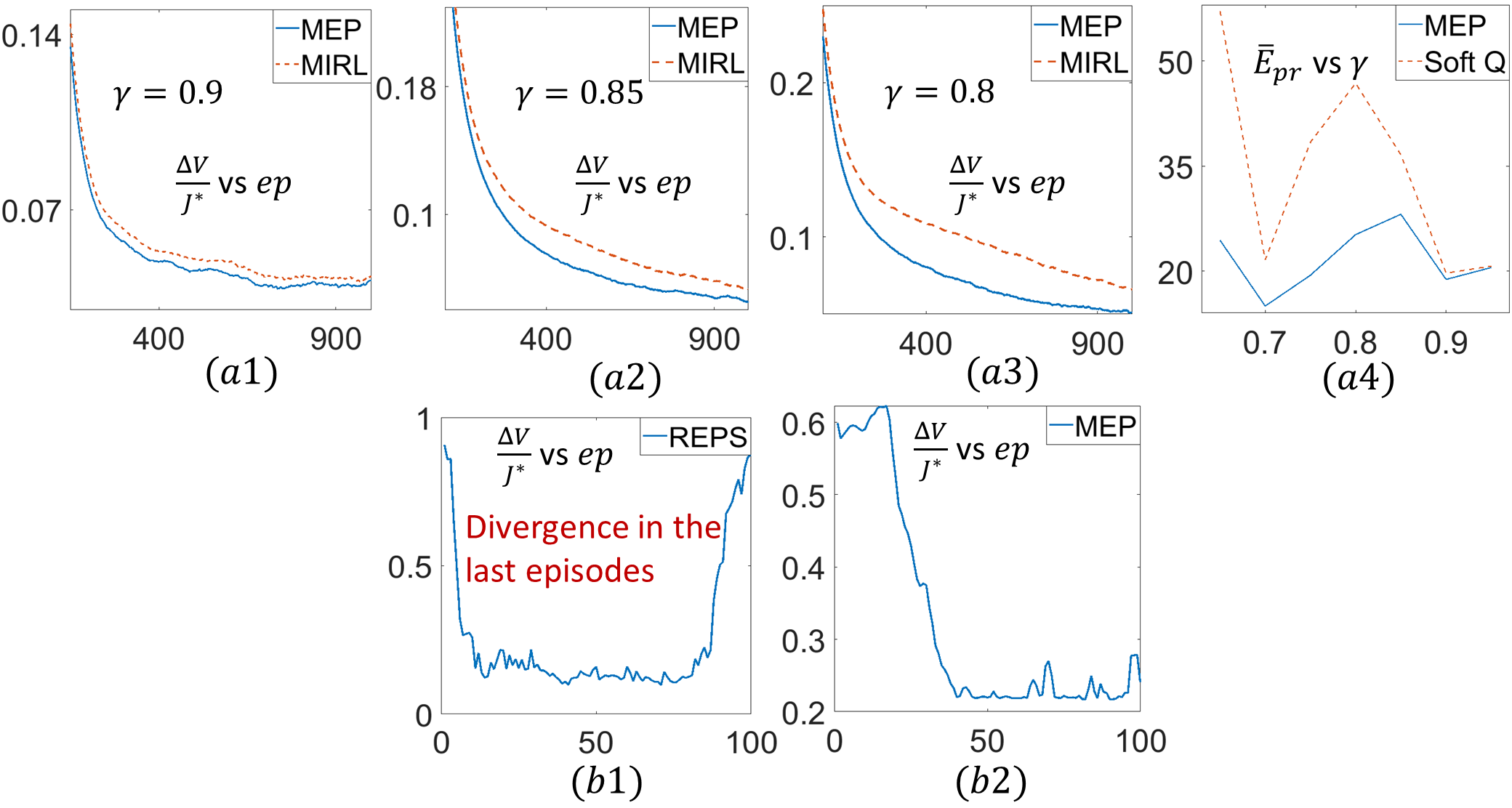}
    \caption{(b1)-(b2) Demonstrating the effect of noisy environment on the Relative Entropy Policy Search (REPS) method in \cite{peters2010relative}. For the same number of interactions with the noisy Gridworld environment, same amount of greedy exploration the REPS method diverges towards the end episode of the learning process; whereas, the MEP-based algorithm does not diverge. Similar results are seen on comparing REPS with entropy regularized (Soft Q) learning too.}
    \label{fig:CompWithMIRL}
\end{figure}

In Figure \ref{fig:CompWithMIRL}(b1)-(b2) we compare the performance of the relative entropy policy search algorithm (REPS) \cite{peters2010relative} to our MEP-based algorithm in the case of noisy environment. For the purpose of simulation we consider the infinite entropy variant of the double chain environment in Figure \ref{fig: DoubleChainEnvs}. We add Gaussian noise $\mathcal{N}(0,\sigma^2)$ with $\sigma=10$ when $s=4$, $a=0$, $\sigma=5$ when $s=8$, $a=0$, $\sigma=2$ when $s\neq 0$, $a=1$, $\sigma=1$ else. We ensure fairness in comparison by allowing equal number of interactions with the environment for both the algorithms. As is illustrated in Figure \ref{fig:CompWithMIRL}(b1) the algorithm in \cite{peters2010relative} diverges during the last stages of learning owing to the noise in the environment. However, for the same environment our MEP-based algorithm performs better.

\section{Design of Self Organizing Networks}\label{Sec: SelfOrgNetwork}

Here, we consider the scenario of self-organizing network that are responsible to provide good wifi access at all the relevant locations in a hotel lobby where the number of routers and the location of the modem are pre-specified. Figure \ref{fig:SelfOrgNetwork}(a) illustrates the floor plan of a hotel lobby with areas such as reception, manager's office, and cyber space clearly marked. We distribute user nodes over the provided floor plan based on the number of wifi users in each of these areas, and assume the modem to be located near the reception area. Figure \ref{fig:SelfOrgNetwork}(b) demonstrates the user node and modem location on a 2D graph. Considering the objective is to allocate $6$ routers and design communication routes from modem to each user via the routers that maximizes the cumulative signal strength at each user node we design this network both in the model-based scenario as well as model-free setting. For the purpose of simulation we assume that the signal strength is inversely proportional to the square of euclidean distance ({\em Qin, Qiaofeng, et al. "SDN controller placement with delay-overhead balancing in wireless edge networks." IEEE Transactions on Network and Service Management 15.4 (2018): 1446-1459.}). Figure \ref{fig:SelfOrgNetwork}(c) represents the allocated routers and the communication routes in the network for the model-based scenario. Figure \ref{fig:SelfOrgNetwork}(d) represents the allocated routers and the communication routes in the case of model-free setting.

\begin{figure}
    \centering
    \includegraphics[width=0.6\columnwidth]{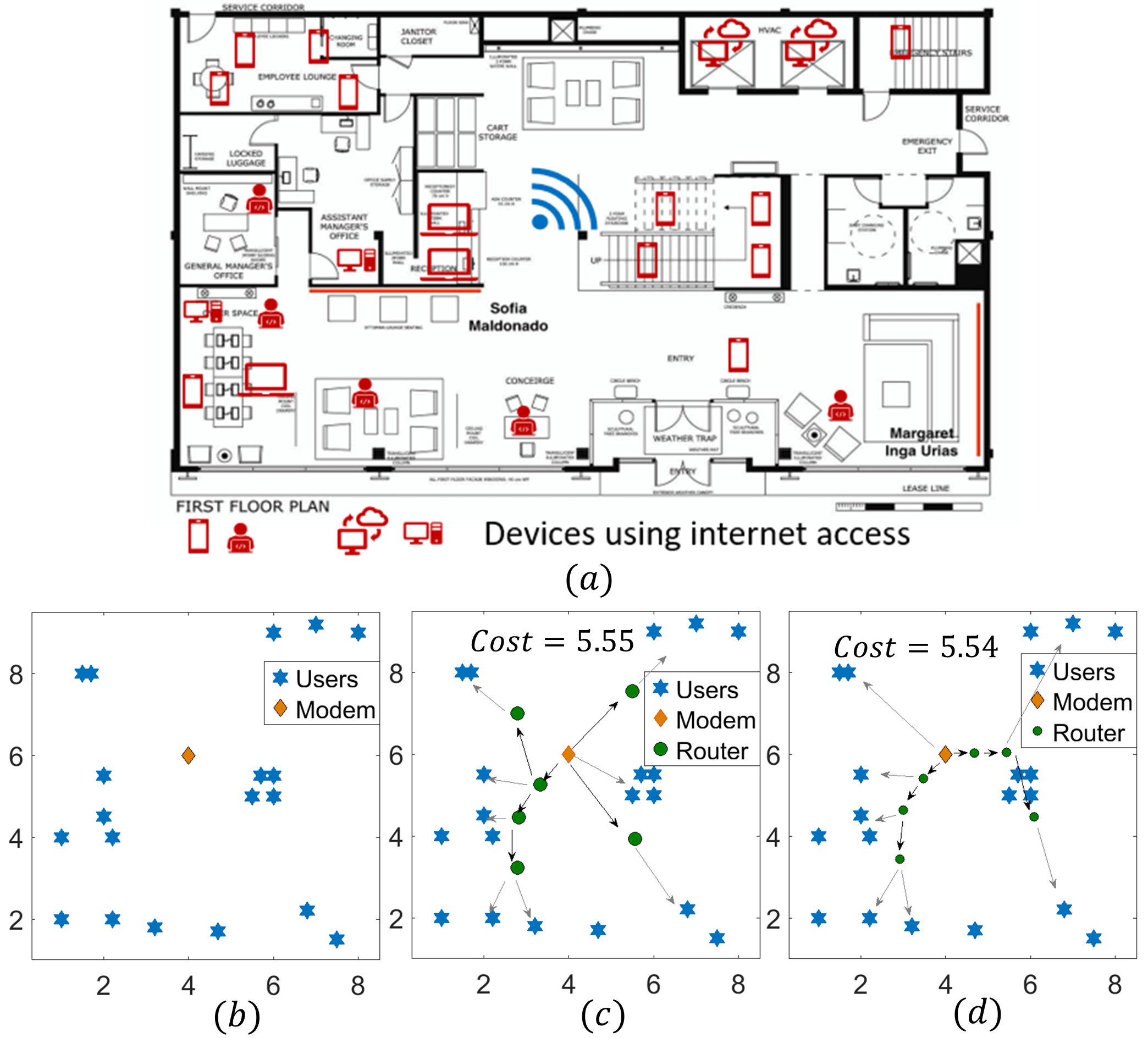}
    \caption{Design of Self Organizing Networks. (a) Floor plan of a hotel lobby along. Areas requiring internet access are marked via appropriate symbols. (b) 2D Schematic of the floor plan with user and modem location marked. (c) Network design with router location and routing. Cost function is square euclidean distance. (d) Network design with router location (parameter) and routing (control policy) in the model-free RL setting.}
    \label{fig:SelfOrgNetwork}
\end{figure}

\section{\textbf{Choice of annealing parameters $\sigma$, $\tau$}} \label{sec: AnnealParaDetails}

{\em Algorithm \ref{alg: Algorithm1}: } We follow a linear schedule $\beta_k = \sigma k$ ($\sigma>0$) in our simulations on model-free RL setting as suggested in the entropy regularized (soft Q) benchmark algorithm in \cite{fox2015taming}. The idea is to anneal the parameter $\beta$ from a small value (at which the policy $\mu_{\beta}^*$ in (\ref{eq: Policy}) is explorative) to a large value (where the policy becomes exploitative). As suggested in \cite{fox2015taming} the parameter $\sigma$ can be obtained by performing initial runs (for a small number of iteration) with different values of $\sigma$, and picking the one that results into lower value function corresponding to the learned policy. This identified value of $\sigma$ for a particular domain can then be re-used in similar domains without the need of performing any initial runs.

{\em Algorithm \ref{alg: Algorithm2} and \ref{alg: Algorithm3}: }The parameter $\tau$ is referred to as the annealing rate and is chosen to be greater than $1$. The resulting schedule $\beta_{k+1}=\tau\beta_k$ gemoterically anneals the Lagrange parameter $\beta$ from a small value $\beta_{\min}\approx 0$ to a large value $\beta_{\max}$ at which the control policy $\mu_{\beta}^*$ in (\ref{eq: Policy}) converges to $0$ or $1$. In the case of parameterized MDPs, this facilitates a homotopy from the convex function $H^{\mu}(s)$ to the (possibly) non-convex $J^{\mu}_{\zeta\eta}(s)$ in (\ref{eq: ParValFunc}), and prevents from getting stuck in poor local minima. For the purpose of simulations in Figure \ref{fig: 5G_SC} we consider the value $\tau\in(1.01, 1.04)$.

A practical method to determine appropriate $\tau$ can be as follows. Start with a large (for instance $\tau\approx 1.5$) estimate for $\tau$. If the parameter values obtained in the initial iterations of the algorithm {\em oscillate} a lot then decrease $\tau$. Choose the $\tau$ value where the observed parameter values stop oscillating for the initial iterations of the algorithm. This practical method is rooted in the phenomenon of {\em phase transition} that our Algorithms \ref{alg: Algorithm2} and \ref{alg: Algorithm3} undergo. We illustrate this phenomenon further below. 

The Algorithms \ref{alg: Algorithm2} and \ref{alg: Algorithm3} address the class of parameterized MDPs that simultaneously solve for an unknown parameter and the control policy. These problems are akin to the Facility Location Problem (FLP) addressed in \cite{rose1991deterministic}. In particular, \cite{rose1991deterministic} develops a Maximum Entropy Principle framework to determine the location of the facilities (parameter), and associate each user node to the closest facility (control policy). In the resulting algorithm (popularly known as Deterministic Annealing (DA)) it is observed that the solution changes significantly only at certain {\em critical values} of $\beta=\beta_{cr}$ that correspond to the instances of phase transition. At other values of $\beta$, the solution does not change much \cite{sharma2011entropy}. It has been observed that a geometric law $\beta_{k+1}=\tau\beta_k$ with $\tau=1+\epsilon>1$ to anneal $\beta$ suffices to capture the changes in the solution that occur during the phase transition. Thus, with reference to the method for choosing a value of $\tau$, if the initial iterations result in high variation in parameter values then possibly some phase transitions are getting skipped and the user needs to anneal slower (i.e., reduce $\tau$) to capture all the phase transitions appropriately.

\section{\textbf{List of Symbols}}\label{LOS}
\begin{center}
\begin{tabular}{c c | c c}
  $x_t$ & state at time $t$ & $u_t$ & action at time $t$\\
  $\mathcal{S}$ & State Space & $\mathcal{A}$ & Action Space\\
  $\mu$ & control policy & $p(s'|s,a)$ & state transition probability\\
  $\delta$ & cost-free termination state &$\gamma$ & discount factor \\
  $J^{\mu}(s)$ & value function under policy $\mu$ & $\mu^*$ & optimal control policy\\
  $\omega$ & path of the MDP & $p_{\mu}(\omega|s)$ & ditribution over the paths\\
  $\Gamma$ & Set of proper policies & $H^{\mu}(s)$ & Shannon Entropy of the distribution $\{p_{\mu}(\omega|s)\}$\\
  $\beta$ & annealing (Lagrange) parameter & $J_0$ & constant value\\
  $V_{\beta}^{\mu}(s)$ & Lagrangian or free-energy & $c_{x_tx_{t+1}}^{u_t}$ & short for $c(x_t,u_t,x_{t+1})$\\
  $\mu_{u_t|x_t}$ & short for $\mu(u_t|x_t)$ & $p_{x_tx_{t+1}}^{u_t}$ & short for $p(x_{t+1}|x_t,u_t)$\\
  $\bar{c}_{ss'}^a$ & short for $c(s,a,s')+\frac{\gamma}{\beta}\log p(s'|s,a)$ & $\lambda_s$ & Lagrange parameter in Lemma 2\\
  $\mu_{\beta}^*(a|s)$ & optimal policy under MEP & $\Lambda_{\beta}(s,a)$ & state-action value function\\
  $V_{\beta}^*$ & optimal value function & $\xi$ & defines weighted norm for contraction mapping\\
  $\|\cdot\|_{\xi}$ & weighted norm & $\alpha$ & contraction constant\\
  $\Psi_{t+1}$ & estimate of $\Lambda_{\beta}$ at time $t$ &  $\nu_t(x_t,u_t)$ & learning rate at time $t$\\
  $\beta_{\min}$ & minimum value of $\beta$ & $\beta_{\max}$ & maximum value of $\beta$\\
  $N$ & number of episodes & $\tau$ & annealing rate for $\beta$\\
  $H_d^{\mu}(s)$ & discounted Shannon entropy & $\alpha^t$ & discount at $t$ for $H_d^{\mu}$\\
  $V_{\beta,I}^{\mu}(s)$ & free-energy for infinite entropy case of MDP & $\hat{c}_{x_tx_{t+1}}^{u_t}$ & short for $c_{x_tx_{t+1}}^{u_t}+\frac{\gamma^t}{\beta\alpha^t}\log p_{x_tx_{t+1}}^{u_t}$\\
  $\mu_{\beta,I}^*$ & optimal policy for infinite entropy MDP & $\Phi_{\beta}(s,a)$ & state-action value function\\
  $V_{\beta,I}^*$ & optimal free-energy function & $\zeta=\{\zeta_s\}$ & set of unknown state parameters\\ 
  $\eta=\{\eta_a\}$ & set of unknown action parameters & $J_{\zeta\eta}^{\mu}$ & value function parameterized MDPs \\
\end{tabular}
\end{center}
\begin{center}
\begin{tabular}{c c|c c}
$x_t(\zeta)$& state at time $t$ with parameter $\zeta_{x_t}$ & $u_t(\eta)$ & action at time $t$ with parameter $\eta_{u_t}$\\
  $\mu_{\beta,\zeta\eta}^*$ & optimal control policy & $V_{\beta,\zeta\eta}^*$ & optimal value function\\
  $G_{\zeta_s}^{\beta}(s')$ & derivative $\partial V_{\beta,\zeta\eta}^*(s')/\partial \zeta_s$ & $G_{\eta_a}^{\beta}(s')$ & derivative $\partial V_{\beta,\zeta\eta}^*(s')/\partial \eta_a$\\
  $K_{\zeta_s}^{\beta}(s',a')$ & derivatives $G_{\zeta_s}^{\beta}(s')=\sum_{a'}\mu_{a'|s'}K_{\zeta_s}^{\beta}(s',a')$ & $L_{\eta_a}^{\beta}(s',a')$ & derivatives $G_{\eta_a}^{\beta}(s')=\sum_{a'}\mu_{a'|s'}L_{\eta_a}^{\beta}(s',a')$\\
  $K_{\zeta_s}^{t+1}(x_t,u_t)$ & learned estimate of $K_{\zeta_s}$ at $t$ & $B,C$ & finite constants
\end{tabular}
\end{center}

\end{document}